\documentclass{article} 
\usepackage{iclr2027_conference,times}

\usepackage{amsmath,amsfonts,bm}

\def\eqref#1{equation~\ref{#1}}

\def\1{\bm{1}}

\DeclareMathAlphabet{\mathsfit}{\encodingdefault}{\sfdefault}{m}{sl}
\SetMathAlphabet{\mathsfit}{bold}{\encodingdefault}{\sfdefault}{bx}{n}

\usepackage{hyperref}
\usepackage{url}
\usepackage{graphicx}
\usepackage{amssymb}
\usepackage{amsthm}  
\usepackage{booktabs}
\usepackage{multirow}
\usepackage{xcolor}
\usepackage{colortbl}
\usepackage{makecell}
\usepackage{subcaption}
\usepackage{algorithm}
\usepackage{algpseudocode}
\usepackage{enumitem}

\definecolor{citeNavy}{HTML}{1A5276}   
\definecolor{linkNavy}{HTML}{154360}   
\definecolor{webCyan}{HTML}{0E6655}    

\hypersetup{
    colorlinks=true,
    breaklinks=true,
    citecolor=citeNavy,
    linkcolor=linkNavy,
    urlcolor=webCyan
}

\definecolor{lunarBlue}{HTML}{EBF3FB}    
\definecolor{lunarHeader}{HTML}{D6E8F6}  

\newtheorem{theorem}{Theorem}

\newtheorem{proposition}{Proposition}

\theoremstyle{definition}

\theoremstyle{remark}

\title{When to Intervene? State-Aware Sparse Manipulation in Federated Reinforcement Learning}

\author{
\textbf{Shutong Zheng}$^{2}$, \textbf{Sijia Chen}$^{1,*}$ \\
$^{1}$The Hong Kong University of Science and Technology (Guangzhou) \\
$^{2}$Sun Yat-sen University \\
\texttt{zhengsht29@mail2.sysu.edu.cn}, \quad \texttt{sijiachen@hkust-gz.edu.cn}
}
\iclrfinalcopy 
\begin{document}

\maketitle

\begin{abstract}
Federated reinforcement learning (FRL) enables distributed agents to collaboratively train decision-making policies, but its decentralized training process also exposes global policy learning to Byzantine manipulation. Existing poisoning attacks primarily focus on \textbf{how} to construct malicious updates, while trajectory-level intervention timing remains largely implicit. In sequential decision making, however, where an intervention is applied can alter subsequent trajectories and learning signals. Through controlled experiments, we find that changing the selected trajectory states materially alters attack efficacy even when the malicious-update construction is fixed. We therefore identify \textbf{when} as a distinct attack dimension and introduce the \underline{V}iability-constrained \underline{B}ehavioral \underline{S}teering \underline{A}ttack (\textsc{V-BSA}), which uses local policy uncertainty to select sparse intervention states and applies envelope-constrained behavioral steering. Across discrete-action benchmarks, \textsc{V-BSA} achieves substantial degradation against robust aggregators and ensemble defenses with only a fraction of the interventions used by dense poisoning, while revealing task- and aggregation-dependent boundaries. Overall, our results highlight intervention timing as a distinct dimension of sequential robustness in FRL. The code is available at \url{https://github.com/Yodeesy/V-BSA}.
\end{abstract}

\begin{figure}[h]
\centering
\includegraphics[width=\linewidth]{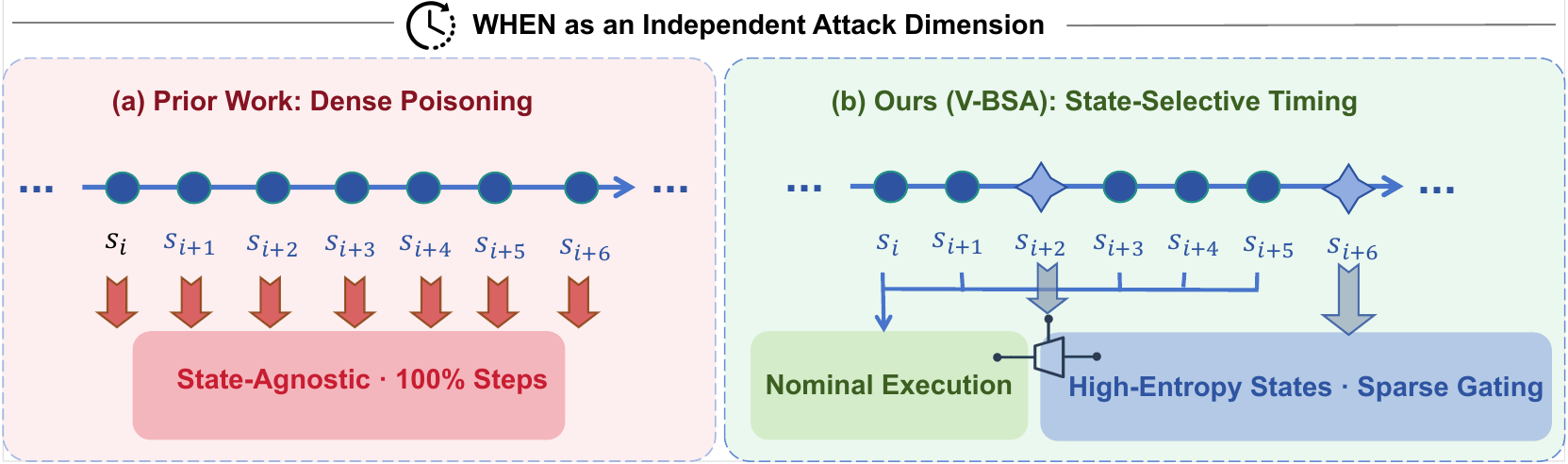}
\vspace{-1.2em}
\caption{\textbf{Conceptual comparison of intervention timing in FRL poisoning.}
\textbf{(a) Prior work} leaves trajectory-level intervention timing implicit, effectively defaulting to dense, state-agnostic intervention across rollout steps.
\textbf{(b) Ours (\textsc{V-BSA})} treats \textsc{when} as an explicit attack dimension: compromised agents maintain nominal execution on routine states and selectively intervene at states identified by policy-entropy gating.}
\label{fig:teaser}
\vspace{-0.8em}
\end{figure}

\section{Introduction}
\label{sec:intro}

Federated reinforcement learning (FRL) enables distributed agents to collaboratively optimize policies while keeping local trajectory data decentralized \citep{mcmahan2017communication, jin2022federated, khodadadian2022federated}. This decentralized process, however, exposes global policy learning to Byzantine clients that can manipulate local learning signals or submitted updates \citep{blanchard2017machine, fang2020local}. Byzantine-robust aggregation has therefore become central to secure FRL, ranging from coordinate-wise statistics such as Median and Trimmed Mean to FRL-specific filtering and recent ensemble-based architectures \citep{yin2018byzantine, fan2021fault, fang2025provably}.

Recent poisoning attacks in FRL have explored multiple attack surfaces, including local environment and reward manipulation \citep{ma2023local,ma2024reward}, malicious policy-update construction under Byzantine settings \citep{fan2021fault,fang2025provably}, and selective perturbation of important model components \citep{zhang2025cpp}. Despite these different mechanisms, attack-oriented approaches primarily focus on \textbf{how} malicious signals or updates are constructed. In sequential decision making, the consequence of a manipulation can vary with the states at which it is applied, as interventions may alter subsequent trajectories, state visitation, and future learning signals. Thus, under a limited intervention budget, the granularity at which interventions are scheduled may itself affect attack leverage.

This leaves a trajectory-level design choice largely implicit in existing FRL poisoning: which states along a local trajectory should receive intervention. Making this choice explicit is important for understanding whether attack efficacy depends only on the construction of the malicious update, or also on where that update is induced along the trajectory.

Motivated by this consideration, we conduct controlled experiments that hold the malicious-update construction (\textbf{how}) fixed while varying only the trajectory-level intervention schedule (\textbf{when}). We observe that changing the selected state subset materially alters downstream attack degradation across evaluation dynamics, even under the same update construction and intervention budget. This empirical observation motivates us to formalize \textbf{when} as a distinct attack dimension under a strictly local, defense-agnostic threat model, explicitly separating trajectory-level state selection (WHEN) from malicious-update construction (HOW).

Based on this observation, we introduce the \textbf{Viability-Constrained Behavioral Steering Attack} (\textsc{V-BSA}). The WHEN component uses locally observable policy entropy to select sparse intervention states, while the HOW component applies target-directed behavioral steering, matches its update scale to the nominal local update, and constrains the result within a client-local empirical update envelope. By explicitly separating trajectory-level state selection from malicious-update construction, \textsc{V-BSA} concentrates a limited intervention budget on selected trajectory states without relying on peer updates or server-side defense information.

We evaluate \textsc{V-BSA} across three discrete-action control environments and multiple Byzantine-robust aggregation settings. The results demonstrate substantial degradation under both conventional robust aggregators and ensemble-based defenses, while also revealing clear task- and aggregation-dependent transfer boundaries. Together with the controlled WHEN comparison, these results show that intervention timing provides a distinct perspective for understanding sequential robustness in FRL.

Our main contributions are summarized as follows:
\begin{itemize}[leftmargin=1.2em, itemsep=1pt, topsep=1pt]
\item \textbf{WHEN as a distinct attack dimension.} We explicitly separate trajectory-level intervention timing (WHEN) from malicious-update construction (HOW) in Byzantine FRL and formalize the resulting two-slot attack process.
\item \textbf{Local-only state-aware sparse manipulation.} We propose \textsc{V-BSA}, which selects intervention states from local policy statistics and performs viability-constrained behavioral steering under a strictly local, defense-agnostic threat model.
\item \textbf{Empirical characterization.} Across three discrete-action environments and multiple aggregation settings, we demonstrate the effect of state-selective intervention and characterize its task- and aggregation-dependent transfer boundaries.
\end{itemize}

\section{Related Work}
\label{sec:related_work}

\paragraph{Poisoning Attacks in FRL.}
Model poisoning in federated learning (FL) manipulates client updates to steer the global model \citep{fang2020local}, either by maximizing divergence from benign updates \citep{fang2020local, shejwalkar2021manipulating} or bounding perturbations within empirical dispersion envelopes \citep{baruch2019little}. FRL introduces sequential, Markovian sampling dynamics that distinguish local update distributions from static FL \citep{khodadadian2022federated}. Recent FRL attacks target different attack surfaces, including local-environment or reward manipulation \citep{ma2023local,ma2024reward}, malicious policy-update construction against Byzantine-robust aggregation \citep{fang2025provably}, and selective perturbation of critical model components \citep{zhang2025cpp}. These approaches primarily focus on malicious signal or update construction (\textsc{how}), leaving trajectory-level intervention timing implicit.

\vspace{-0.4em}

\paragraph{Byzantine-Robust Aggregation and Ensemble Defenses.}
Byzantine-robust learning mitigates malicious updates through filtering and robust statistics, including Krum \citep{blanchard2017machine}, Bulyan \citep{el2018hidden}, coordinate-wise Median and Trimmed Mean \citep{yin2018byzantine}, and geometric medians \citep{chen2017distributed}, supplemented by worker replication schemes such as DETOX \citep{rajput2019detox}. In FRL, FedPG-BR evaluates policy gradients via directional and magnitude alignment against a robust reference \citep{fan2021fault}, while recent work examines the empirical sufficiency of robust aggregation under heterogeneous dynamics \citep{fang2025we} and client grouping with multi-policy ensemble decisions \citep{fang2025provably}. These defenses operate primarily on submitted updates or resulting policies, without explicitly accounting for the trajectory states that generate them.

\vspace{-0.4em}

\paragraph{Temporal Selectivity and Sparse Interventions.}
Temporal selectivity has been studied at different operational scales. In FL, recent attacks exploit communication-round selection or multi-round consistency \citep{yan2024enhancing,xie2025model,coba2025,ye2026less}, while temporal defenses audit cross-round update similarity to mitigate persistent poisoning \citep{fung2018mitigating}. In single-agent RL, adaptive reward poisoning and state-aware adversarial perturbations target selected decision steps \citep{lin2017tactics,zhang2020adaptive,cui2024badrl,zhang2025state}. These lines of work, however, do not jointly address federated aggregation and intra-round trajectory-level intervention in FRL. To our knowledge, prior FRL poisoning work does not explicitly formulate trajectory-level intervention timing as a distinct attack dimension.

\vspace{-0.4em}

\section{Preliminaries and Motivation}
\label{sec:preliminaries_motivation}

\begin{figure}[t]
  \centering
  \includegraphics[width=0.98\linewidth]{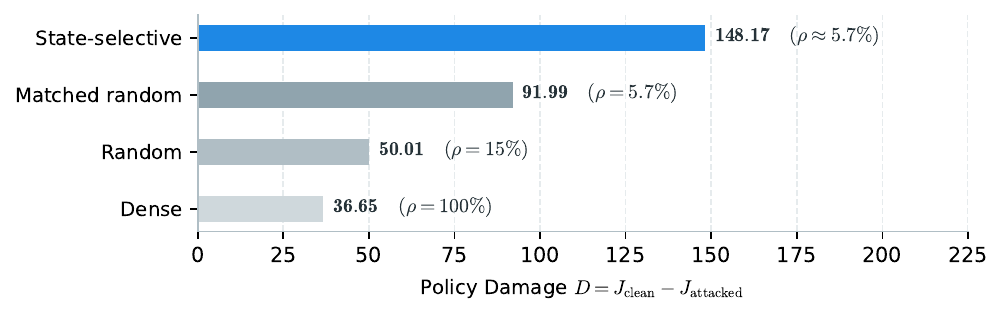}
  \vspace{-0.6em}
    \caption{\textbf{Pilot evaluation of intervention timing.} Terminal degradation \(D\) on LunarLander under FedPG-BR with fixed update construction across schedules (\(\rho\)). Entropy gating induces substantially higher degradation than dense and budget-matched baselines.}
  \label{fig:pilot_observation}
  \vspace{-0.8em}
\end{figure}

\subsection{Federated Reinforcement Learning}
\label{sec:preliminaries}

We consider a federated reinforcement learning (FRL) system with a central server and \(N\) distributed agents. Agent \(i\) interacts with a local MDP \(\mathcal{M}_i=(\mathcal{S},\mathcal{A},\mathcal{P}_i,\mathcal{R}_i,\gamma,\rho_i)\) under a shared parameterized policy \(\pi_\theta(a\mid s)\), where \(\mathcal{S}\) and \(\mathcal{A}\) denote the state and action spaces. Its expected discounted return is \(J_i(\theta)=\mathbb{E}_{\tau\sim(\pi_\theta,\mathcal{M}_i)}[\sum_{k=0}^{\infty}\gamma^k\mathcal{R}_i(s_k,a_k)]\), and the federated learning objective is to optimize the global performance \(J(\theta)=\frac{1}{N}\sum_{i=1}^{N}J_i(\theta)\).

Training proceeds over communication rounds \(t=1,\dots,T\). At round \(t\), the server broadcasts the global parameter \(\theta^t\), each agent executes local rollouts to compute an update \(\Delta\theta_i^t\), and the server updates the global policy via \(\theta^{t+1}=\theta^t+\operatorname{AGR}(\{\Delta\theta_i^t\}_{i=1}^{N})\). The aggregation rule \(\operatorname{AGR}\) may be FedAvg or a Byzantine-robust alternative.

\subsection{Motivation}
\label{sec:motivation}

To test whether finer-grained trajectory-level scheduling provides attack leverage beyond intervention frequency, we conduct a controlled pilot study on LunarLander under FedPG-BR aggregation. We hold the malicious-update construction (\textbf{how}) fixed and vary only the intervention schedule (\textbf{when}): dense (\(\rho=100\%\)), random (\(\rho=15\%\)), budget-matched random (\(\rho=5.7\%\)), and entropy-based state selection (\(\rho\approx5.7\%\)).

Figure~\ref{fig:pilot_observation} shows that intervention frequency alone does not determine attack effectiveness. Dense intervention yields \(D=36.65\), while random intervention at \(\rho=15\%\) yields \(D=50.01\). Entropy-based selection reaches \(D=148.17\) with only \(\approx 5.7\%\) of states intervened on, corresponding to \(\approx 17.5\times\) fewer interventions than the dense baseline. More importantly, under the matched \(5.7\%\) budget, entropy-based selection achieves \(D=148.17\) versus \(D=91.99\) for random selection. Thus, at a comparable intervention budget and identical update construction, changing which trajectory states are selected materially alters downstream policy degradation, providing direct evidence that trajectory-level scheduling is an independent attack dimension.

\section{Methodology}
\label{sec:methodology}

\subsection{Problem Setting}
\label{sec:problem_setting}

\paragraph{Threat Model and Objective.}
We consider Byzantine FRL with \(N\) agents, of which \(M\) are compromised, with corruption ratio \(\beta=M/N\). At round \(t\), each compromised agent receives \(\theta^t\) and observes only its local rollout \(\mathcal{D}_i^t=\{(s_{i,k}^t,a_{i,k}^t,r_{i,k}^t)\}_{k=1}^{K}\) and historical updates \(\Delta\theta_i^{1:t-1}\). The adversary operates under a \textbf{strictly local, defense-agnostic} information boundary, observing neither peer updates nor server-side aggregation mechanics. Compromised agents execute nominal local rollouts prior to attack synthesis, thereby obtaining the benign update \(g_i^t=\Delta\theta_{i,\mathrm{clean}}^t\). The objective is to maximize terminal degradation \(D\triangleq J(\theta_{\mathrm{clean}})-J(\theta_{\mathrm{attacked}})\) under bounded intervention budgets and locally constrained updates.

\paragraph{WHEN--HOW Factorization.}
FRL operates at two temporal levels: agents make sequential decisions at intra-round trajectory steps \(k\), while local updates are aggregated across communication rounds \(t\). We explicitly separate intervention timing from update construction:
\begin{equation}
\text{\textbf{WHEN:}}\ \mathcal{G}_{\mathrm{when}}(\cdot;\theta^t):\mathcal{S}\rightarrow\{0,1\}, 
\qquad
\text{\textbf{HOW:}}\ \mathcal{H}_{\mathrm{how}}:(g_i^t,\mathcal{D}_{i,\mathrm{when}}^t)\rightarrow\Delta\theta_i^{t,\mathrm{mal}}.
\label{eq:when_how_operators}
\end{equation}
The WHEN operator filters rollouts into a sparse intervention set \(\mathcal{D}_{i,\mathrm{when}}^t=\{(s,a,r)\in\mathcal{D}_i^t \mid \mathcal{G}_{\mathrm{when}}(s;\theta^t)=1\}\). Conventional dense poisoning corresponds to \(\mathcal{G}_{\mathrm{when}}\equiv 1\), leaving timing unconstrained. In contrast, \textsc{V-BSA} makes timing explicit via policy-entropy gating and instantiates \(\mathcal{H}_{\mathrm{how}}\) via viability-constrained behavioral steering (Figure~\ref{fig:framework}).

\subsection{Intervention Timing}
\label{sec:when_theory}

The WHEN operator identifies locally susceptible trajectory states without access to peer or server information.

\paragraph{Local Behavioral Susceptibility.}
For a categorical policy \(\pi_\theta(\cdot\mid s)=\operatorname{softmax}(z_\theta(s))\), parameter perturbation \(\delta\theta\) induces behavioral change measured by \(\mathbb{D}_{\mathrm{KL}}(\pi_\theta(\cdot\mid s)\Vert\pi_{\theta+\delta\theta}(\cdot\mid s))\). Following foundational trust-region and natural policy-gradient literature \citep{kakade01fim,schulman2015trust}, for locally \(C^3\) policies with bounded third derivatives,
\begin{equation}
\mathbb{D}_{\mathrm{KL}}\!\left(\pi_\theta\Vert\pi_{\theta+\delta\theta}\right)
=
\frac{1}{2}\delta\theta^\top F(s)\delta\theta
+\mathcal{O}(\|\delta\theta\|_2^3),
\qquad
F(s)=J_z(s)^\top M(p)J_z(s),
\label{eq:kl_second_order}
\end{equation}
where \(J_z(s)=\nabla_\theta z_\theta(s)\), \(p=\pi_\theta(\cdot\mid s)\), and \(M(p)=\operatorname{Diag}(p)-pp^\top\). Worst-case local susceptibility is \(\chi_\epsilon(s)=\frac{\epsilon^2}{2}\lambda_{\max}(F(s))+\mathcal{O}(\epsilon^3)\). Because the attacker cannot anticipate which parameter directions will be emphasized by an unknown aggregation rule, we adopt the isotropic second-order surrogate \(\bar{\chi}^{(2)}_\sigma(s)=\frac{\sigma^2}{2}\operatorname{Tr}(F(s))\).

\paragraph{Idealized Selection and Entropy Proxy.}
Under cardinality budget \(B\), total additive susceptibility is maximized by selecting the \(B\) states with largest \(\operatorname{Tr}(F(s))\). This oracle ranking serves strictly as an analytical benchmark; \textsc{V-BSA} operates online without buffering trajectories or evaluating Jacobians. To obtain an output-only proxy, consider the sufficient centered-isotropy condition \(J_z(s)J_z(s)^\top=cP_\perp\), where \(P_\perp=I_{|\mathcal A|}-\frac{1}{|\mathcal A|}\mathbf{1}\mathbf{1}^\top\) and \(c>0\) is state-independent. Because \(M(p)\mathbf{1}=0\), we obtain \(\operatorname{Tr}(F(s))=c(1-\|p\|_2^2)\).

\begin{theorem}[Entropy--Susceptibility Ranking under Centered Isotropy]
\label{thm:entropy_susceptibility}
For a binary policy \(p=(q,1-q)\) under centered isotropy with state-independent \(c>0\), Shannon entropy \(H(p)\) and Fisher trace \(\operatorname{Tr}(F(s))\) induce the identical state ranking: \(H(p_1) > H(p_2) \iff \operatorname{Tr}(F(s_1)) > \operatorname{Tr}(F(s_2))\).
\end{theorem}
\emph{Proof sketch.} For binary actions, \(\operatorname{Tr}(F(s))=2cq(1-q)\). Both \(H(q)\) and \(2q(1-q)\) are symmetric about \(q=1/2\) and strictly monotonically decreasing in \(|q-1/2|\), inducing identical orderings over state susceptibilities (complete derivation in Appendix~\ref{app:proof_theorem_1}).

For multi-action policies, exact ranking equivalence does not generally hold, so entropy serves as a practical, backpropagation-free proxy:
\begin{equation}
\mathcal{G}_{\mathrm{when}}(s;\theta^t)
=
\mathbf{1}\!\left[
\frac{H(\pi_{\theta^t}(\cdot\mid s))}{\log|\mathcal A|}
>\tau
\right],
\qquad \tau=0.9,
\label{eq:entropy_gate}
\end{equation}
where \(\tau\) is an \emph{a priori} fixed threshold rather than tuned to an intervention budget, inducing an endogenous intervention rate. Entropy captures output dispersion via \(M(p)\) but not representation sensitivity in \(J_z(s)\), motivating our audit in Section~\ref{sec:when_audit}. Continuous-action boundaries are discussed in Appendix~\ref{app:continuous_study}.

\begin{figure*}[t]
  \centering
  \includegraphics[width=\textwidth]{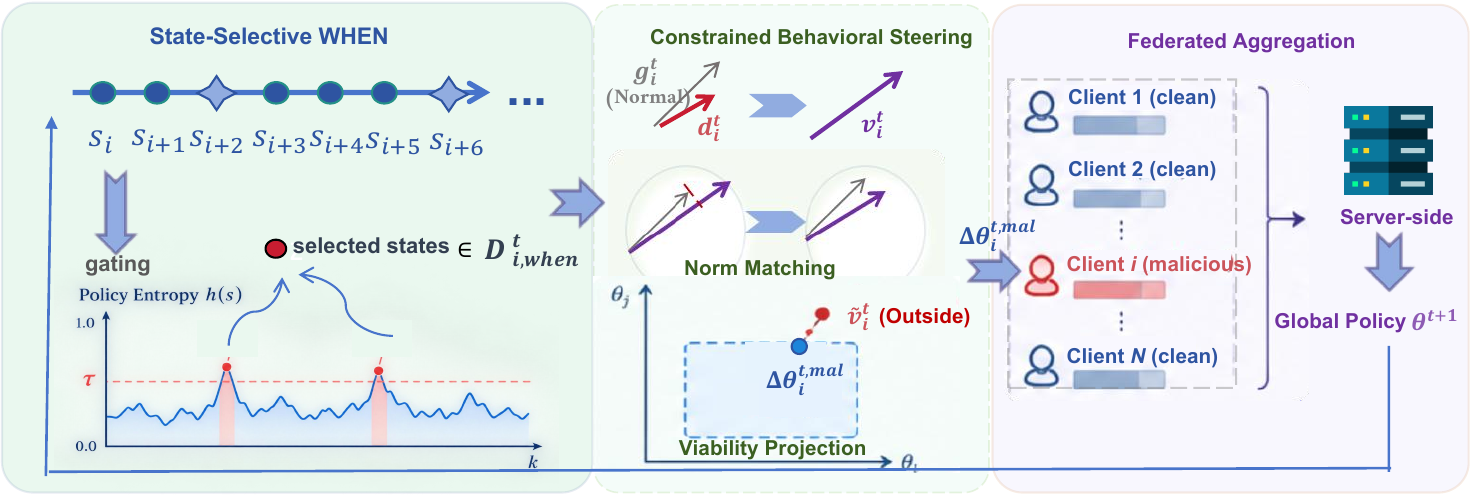}
  \vspace{-0.6em}
  \caption{\textbf{Overview of \textsc{V-BSA}.} The attack factorizes into trajectory-level \textsc{when} selection and parameter-space \textsc{how} construction: an entropy gate selects sensitive states \(\mathcal{D}_{i,\mathrm{when}}^t\), from which a targeted update is synthesized, norm-matched, and constrained by the local empirical envelope \(\mathcal{B}_i^t\) before server-side aggregation.}
  \label{fig:framework}
  \vspace{-0.8em}
\end{figure*}

\subsection{Viability-Constrained Behavioral Steering}
\label{sec:how_formulation}

\paragraph{Behavioral Steering and Scale Matching.}
Given \(\mathcal{D}_{i,\mathrm{when}}^t\), the HOW operator constructs a target-directed update via
\begin{equation}
\mathcal{L}_{\mathrm{steer}}
=
-\frac{1}{|\mathcal{D}_{i,\mathrm{when}}^t|}
\sum_{(s,\cdot,\cdot)\in\mathcal{D}_{i,\mathrm{when}}^t}
\log\pi_\theta(a_{\mathrm{targ}}^{(i)}\mid s),
\qquad
d_i^t=-\nabla_\theta\mathcal{L}_{\mathrm{steer}}(\theta^t),
\qquad
v_i^t=g_i^t+d_i^t .
\label{eq:steering_construction}
\end{equation}
followed by global \(\ell_2\) norm matching \(\tilde v_i^t = v_i^t \cdot (\|g_i^t\|_2 / \max(\|v_i^t\|_2,\epsilon))\). If \(|\mathcal{D}_{i,\mathrm{when}}^t|=0\), the client submits its nominal update \(g_i^t\). Target actions are fixed per compromised subgroup, and evaluations use unit steering weight.

\paragraph{Local Update Constraint.}
The client maintains EMA statistics \((\hat\mu_i^t,\hat\sigma_i^t)\) from historical benign updates to define the client-local empirical envelope:
\begin{equation}
\mathcal{B}_i^t
=
\prod_{j=1}^{d}
\left[
\hat\mu_{i,j}^t-\lambda_{\mathrm{box}}\hat\sigma_{i,j}^t,\,
\hat\mu_{i,j}^t+\lambda_{\mathrm{box}}\hat\sigma_{i,j}^t
\right],
\qquad
\Delta\theta_i^{t,\mathrm{mal}}
=
\operatorname{Proj}_{\mathcal B_i^t}(\tilde v_i^t),
\label{eq:viability_projection}
\end{equation}
where projection is coordinate-wise clipping. We set \(\alpha_{\mathrm{ema}}=0.1\) and \(\lambda_{\mathrm{box}}=0.85\). All quantities rely strictly on local observables; the envelope serves as a compatibility constraint rather than a formal defense-bypass guarantee.

\begin{algorithm}[t]
\caption{\textsc{V-BSA}: Local Malicious Update Synthesis at Round \(t\)}
\label{alg:vbsa}
\small
\begin{algorithmic}[1]
  \State \textbf{Input:} Policy \(\theta^t\), rollouts \(\mathcal{D}_i^t\), EMA stats \((\hat{\mu}_i^t, \hat{\sigma}_i^t)\), target action \(a_{\mathrm{targ}}^{(i)}\), threshold \(\tau\), radius \(\lambda_{\mathrm{box}}\).
  \State Compute nominal local update \(g_i^t = \Delta\theta_{i,\mathrm{clean}}^t\) on local rollout \(\mathcal{D}_i^t\).
  \State \textbf{[WHEN]} Filter sensitive states: \(\mathcal{D}_{i,\mathrm{when}}^t \leftarrow \{(s, a, r) \in \mathcal{D}_i^t \mid H(\pi_{\theta^t}(\cdot \mid s)) / \log|\mathcal{A}| > \tau\}\).
  \State \textbf{if} \(|\mathcal{D}_{i,\mathrm{when}}^t| = 0\) \textbf{then return} \(\Delta\theta_i^{t,\mathrm{mal}} \leftarrow g_i^t\) \Comment{Submit nominal update if inactive}
  \State \textbf{[HOW: Steering]} Compute signed direction \(d_i^t \leftarrow -\nabla_\theta \mathcal{L}_{\mathrm{steer}}(\theta^t)\) on \(\mathcal{D}_{i,\mathrm{when}}^t\) using \(a_{\mathrm{targ}}^{(i)}\).
  \State Combine candidate perturbation vector: \(v_i^t \leftarrow g_i^t + d_i^t\).
  \State \textbf{[HOW: Scale Matching]} Rescale vector: \(\tilde{v}_i^t \leftarrow v_i^t \cdot (\|g_i^t\|_2 / \max(\|v_i^t\|_2, \epsilon))\).
  \State \textbf{[HOW: Envelope Projection]} Project update: \(\Delta\theta_i^{t,\mathrm{mal}} \leftarrow \operatorname{Proj}_{\mathcal{B}_i^t}(\tilde{v}_i^t)\) via Eq.~\eqref{eq:viability_projection}.
  \State \textbf{Output:} Submit constrained malicious update \(\Delta\theta_i^{t,\mathrm{mal}}\) to server-side aggregation.
\end{algorithmic}
\end{algorithm}

\paragraph{Execution Characteristics.}
The entropy gate reuses policy logits already computed during rollout generation, introducing negligible forward overhead. Let \(K_s = |\mathcal{D}_{i,\mathrm{when}}^t| = \rho K\); the additional steering computation is strictly confined to the selected subset, with backward cost scaling as \(\mathcal{O}(K_s C_{\mathrm{bwd}})\), where \(C_{\mathrm{bwd}}\) denotes the per-state backward cost. Subsequent norm matching, EMA updates, and coordinate-wise projection require \(\mathcal{O}(d)\) parameter-space operations. Consequently, \textsc{V-BSA} incurs only local computational overhead proportional to the empirical intervention rate \(\rho\) and parameter dimension \(d\), without requiring Jacobian evaluation, peer-update exchange, or additional communication rounds.

\section{Experiments}

\subsection{Experimental Setup}
\label{sec:experimental_setup}

\paragraph{Environments.}
We evaluate \textsc{V-BSA} on three discrete-action control benchmarks with diverse dynamical properties: CartPole (\(|\mathcal{A}| = 2\)) \citep{barto1983neuronlike}, Acrobot (\(|\mathcal{A}| = 3\)) \citep{sutton1995generalization}, and LunarLander (\(|\mathcal{A}| = 4\)) \citep{brockman2016openai}. This selection spans both binary and multi-action settings. Detailed environment specifications are provided in Appendix~\ref{app:environments}.

\paragraph{Federated Training and Aggregation.}
We consider a system of \(N = 30\) clients, including \(M = 10\) compromised clients (\(\beta = 1/3\)). For ensemble configurations, clients are partitioned into \(K = 5\) groups of six, with two compromised clients per group. All clients parameterize policies via a two-layer MLP with hidden dimension 64 and Tanh activations. We benchmark against both standard and Byzantine-robust aggregation rules: FedAvg \citep{mcmahan2017communication}, Trimmed Mean and Coordinate-wise Median \citep{yin2018byzantine}, and FedPG-BR \citep{fan2021fault}; in ensemble settings, aggregation is performed independently within each group.

\paragraph{Attack Configuration and Evaluation.}
\textsc{V-BSA} uses a single, fixed set of attack coefficients across all tasks and aggregation rules: \(\tau = 0.9\), EMA coefficient \(\alpha_{\mathrm{ema}} = 0.1\), \(\lambda_{\mathrm{box}} = 0.85\), and unit steering weight, with no per-task or per-aggregator coefficient tuning. Target actions are assigned deterministically: compromised groups use the group index modulo \(|\mathcal{A}|\) on the multi-action benchmarks, while binary-action CartPole uses a single shared target action $0$. This avoids splitting the compromised population across the two available actions, which can induce cancellation during aggregation. Both assignment modes are fixed \emph{a priori} and are not tuned to individual experimental configurations. Unless otherwise stated, each configuration is evaluated over three independent seeds, and we report the mean return over runs. Test performance is measured over 10 evaluation rollouts, and attack effectiveness is quantified by \(D \triangleq J_{\mathrm{clean}} - J_{\mathrm{attacked}}\), where \(J_{\mathrm{clean}}\) and \(J_{\mathrm{attacked}}\) denote the evaluation returns of the unpoisoned baseline and the attacked system, respectively.
where \(J_{\mathrm{clean}}\) and \(J_{\mathrm{attacked}}\) denote the evaluation returns of the unpoisoned baseline and the attacked system, respectively. A larger positive \(D\) indicates greater performance degradation.

\begin{table}[t]
\centering
\small
\setlength{\tabcolsep}{4.2pt}        
\renewcommand{\arraystretch}{1.25}    
\setlength{\aboverulesep}{0pt}       
\setlength{\belowrulesep}{0pt}

\begin{tabular}{l >{\columncolor{lunarBlue}}c c c >{\columncolor{lunarBlue}}c c c}
\toprule
\multirow{2}{*}{\textbf{Aggregation Rule}} 
& \multicolumn{3}{c}{\textbf{Ensemble} (\(K=5, M=10\))} 
& \multicolumn{3}{c}{\textbf{Non-Ensemble} (\(N=30, M=10\))} \\
\cmidrule(lr){2-4} \cmidrule(lr){5-7}
& \multicolumn{1}{>{\columncolor{lunarHeader}}c}{\textbf{LunarLander}} & \textbf{Acrobot} & \textbf{CartPole} 
& \multicolumn{1}{>{\columncolor{lunarHeader}}c}{\textbf{LunarLander}} & \textbf{Acrobot} & \textbf{CartPole} \\
\midrule
FedAvg                 & 79.90       & 120.50     & 3.00       & \(-57.19\) & 351.17 & 1.63   \\
FedPG-BR               & 218.75      & 194.80     & 110.63     & 288.01     & 70.63  & 114.27 \\
Coordinate-wise Median & 175.83      & 151.17     & 8.80       & 434.35     & 284.87 & 42.80  \\
Trimmed Mean           & 173.96      & 323.97     & 66.40      & 193.38     & 312.97 & 5.87   \\
\bottomrule
\end{tabular}
\vspace{0.4em}
\caption{\textbf{Mean degradation \(D\) across aggregations and tasks (higher is more damaging).} Positive values indicate successful terminal policy degradation induced by \textsc{V-BSA}.}
\label{tab:main_results}
\end{table}

\begin{figure}[t]
  \centering
  \begin{subfigure}[b]{0.25\linewidth}
    \centering
    \includegraphics[width=\linewidth]{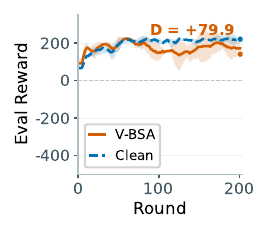}
    \caption{FedAvg}
    \label{subfig:ll_ens_fedavg}
  \end{subfigure}\hfill
  \begin{subfigure}[b]{0.25\linewidth}
    \centering
    \includegraphics[width=\linewidth]{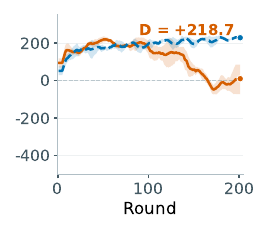}
    \caption{FedPG-BR}
    \label{subfig:ll_ens_fedpg_br}
  \end{subfigure}\hfill
  \begin{subfigure}[b]{0.25\linewidth}
    \centering
    \includegraphics[width=\linewidth]{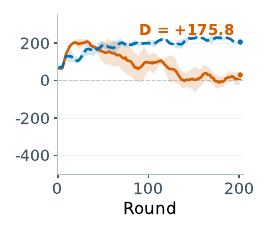}
    \caption{Median}
    \label{subfig:ll_ens_median}
  \end{subfigure}\hfill
  \begin{subfigure}[b]{0.25\linewidth}
    \centering
    \includegraphics[width=\linewidth]{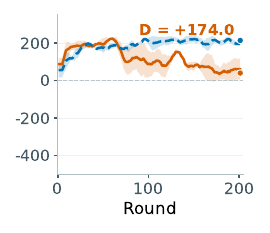}
    \caption{Trimmed Mean}
    \label{subfig:ll_ens_trimmed_mean}
  \end{subfigure}
  \vspace{-0.6em}
    \caption{\textbf{Evaluation dynamics on Ensemble-LunarLander (\(K=5, M=10\)).} Returns of \textsc{Clean} vs.\ \textsc{V-BSA} across four aggregation rules, with shaded bands denoting standard deviation across random seeds.}
  \label{fig:ll_ensemble_curves}
  \vspace{-0.8em}
\end{figure}

\begin{figure}[t]
  \centering
  \includegraphics[width=0.97\linewidth]{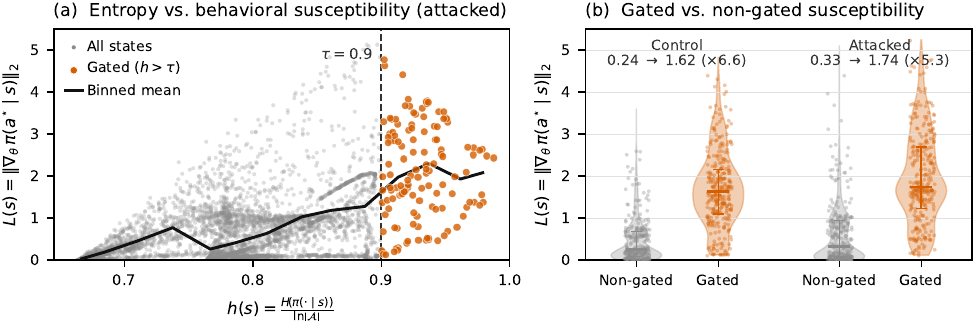}
  \vspace{-1.2em}
  \caption{\textbf{Mechanistic audit of \textsc{when} selection (LunarLander, epoch 261).} \textbf{(a)} Normalized policy entropy \(h(s)\) versus behavioral susceptibility \(L(s)\). \textbf{(b)} Susceptibility distributions for non-gated versus gated states across five client groups.}
  \label{fig:when_audit}
  \vspace{-0.8em}
\end{figure}

\begin{figure}[t]
  \centering
  \includegraphics[width=0.79\linewidth]{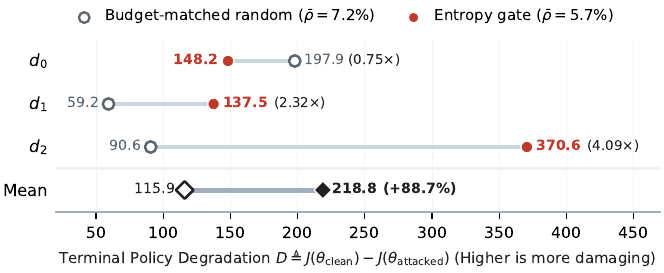}
  \vspace{-0.6em}
    \caption{\textbf{Timing ablation under matched budgets (LunarLander \(\times\) FedPG-BR).} Fixing all \textsc{how} parameters, entropy gating achieves \(1.89\times\) higher mean degradation than budget-matched random selection across evaluated dynamics.}
  \label{fig:when_controlled}
  \vspace{-0.8em}
\end{figure}

\subsection{Experimental Results}
\label{sec:exp_results}

Table~\ref{tab:main_results} reports the mean terminal degradation \(D\) across tasks and aggregation rules, where positive values indicate lower attacked return. Figure~\ref{fig:ll_ensemble_curves} further visualizes the corresponding evaluation dynamics for ensemble LunarLander.

\paragraph{Main Results and Aggregation-Dependent Transfer.}
LunarLander provides clear empirical evidence of effective state-selective manipulation. Under ensemble training, \textsc{V-BSA} yields positive terminal degradation across all four aggregation rules, ranging from \(D=79.90\) (FedAvg) to \(D=218.75\) (FedPG-BR). As shown in Figure~\ref{fig:ll_ensemble_curves}, nominal and attacked evaluation curves progressively separate across communication rounds, with degradation persisting under FedPG-BR filtering.

Cross-aggregation comparison also reveals clear operational boundaries. Under non-ensemble training, FedAvg exhibits a sign reversal (\(D=-57.19\)), where attacked runs attain higher average terminal return than the nominal baseline. We interpret this as an empirical transfer boundary, where the manipulated updates fail to sustain policy degradation. Acrobot exhibits a different pattern, with substantial degradation across aggregation rules, reaching \(D=351.17\) under non-ensemble FedAvg and \(D=323.97\) under ensemble Trimmed Mean. Together, these results illustrate that downstream degradation depends jointly on task dynamics and aggregation.

\paragraph{Binary-Action Cancellation Boundary: CartPole.}
CartPole exposes a structural boundary arising from binary target-steering geometry. Ensemble FedPG-BR yields substantial degradation (\(D=110.63\)), whereas coordinate-wise Median shows little average effect (\(D=8.80\)). Trimmed Mean exhibits an intermediate mean effect (\(D=66.40\)), concentrated on \(d_1\) (\(D=184.3\)) versus \(d_0\) (\(D=0\)) and \(d_2\) (\(D=14.9\)).

For binary actions, the opposing steering directions satisfy
\begin{equation}
  d(a_0;s) + d(a_1;s) = \bigl(1-2\pi_\theta(a_0\mid s)\bigr)\Delta(s),
  \label{eq:cartpole_cancellation}
\end{equation}
where \(\Delta(s)=\nabla_\theta(z_0-z_1)\). Thus, near the maximum-entropy boundary \(\pi_\theta(a_0\mid s)=1/2\), the two target-specific directions cancel. Under alternating target assignment, coordinate-wise order statistics can therefore suppress opposing perturbations, whereas the filtering behavior of FedPG-BR differs because it evaluates updates relative to a robust reference. This geometry explains the observed aggregation divergence; further derivations appear in Appendix~\ref{app:cartpole_aggregation}.

\paragraph{Controlled Comparison of Intervention Timing.}
To isolate the effect of intervention timing (\textbf{when}) under fixed update construction (\textbf{how}), Figure~\ref{fig:when_controlled} compares entropy gating with a budget-controlled uniform random selector on LunarLander under FedPG-BR. The two conditions use the same \textsc{how} components, random seeds, and evaluation dynamics.

Across the three dynamics, entropy gating yields a higher mean degradation (\(\bar{D}=218.75\), median \(148.17\)) than random selection (\(\bar{D}=115.92\), median \(90.64\)), corresponding to a \(1.89\times\) larger mean degradation. The effect is not uniform across dynamics: entropy gating yields \(137.48\) versus \(59.25\) on \(d_1\) and \(370.59\) versus \(90.64\) on \(d_2\), whereas random selection yields higher degradation on \(d_0\) (\(197.86\) versus \(148.17\)). The realized intervention rates are \(5.7\%\) for entropy gating and \(7.2\%\) for random selection on average. These controlled results show that, under fixed update construction and a controlled intervention budget, changing the selected trajectory states can materially alter downstream attack effectiveness, while the magnitude and direction of the effect remain dynamics-dependent.

\paragraph{Mechanistic Audit of WHEN Selection.}
\label{sec:when_audit}
Figure~\ref{fig:when_audit} examines the behavioral susceptibility of visited states in LunarLander. We quantify local susceptibility \(L(s)\) using the perturbation-based behavioral metric defined in Appendix~\ref{app:when_metric}.

As shown in Figure~\ref{fig:when_audit}a, entropy-selected states exhibit a pronounced shift toward higher susceptibility than non-selected states. Across rollout states pooled over the five client groups (Figure~\ref{fig:when_audit}b), median susceptibility increases from \(0.24\) to \(1.62\) (\(6.6\times\)) under nominal policies and from \(0.33\) to \(1.74\) (\(5.3\times\)) under attacked policies. Because a similar separation is already observed under nominal execution, the effect is not solely induced by the attack. These results support policy entropy as a practical output-only proxy for local behavioral susceptibility in the evaluated LunarLander setting, while remaining consistent with the theoretical limitation that multi-action entropy does not generally induce exact Fisher-trace rankings.

\subsection{Temporal Countermeasures and Defense Outlook}
\label{sec:temporal_defenses}

A natural question is whether temporal parameter-space auditing can suppress state-selective manipulation that remains effective against single-round robust filtering. We therefore preregistered and tested two temporal countermeasures on LunarLander: an \emph{epoch-level direction auditor} that scores each client update against a per-round coordinate-median reference, exposing the attack's directional-persistence signature, and a \emph{temporal accumulator} that aggregates rejection evidence across historical updates.

\begin{figure}[t]
    \centering
    \includegraphics[width=0.89\linewidth]{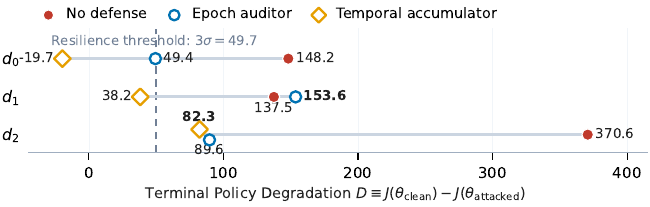}
    \caption{\textbf{Temporal countermeasures against \textsc{V-BSA} on LunarLander.} Terminal degradation \(D\) across dynamics \(d_0, d_1, d_2\) under no defense, the epoch auditor, and the temporal accumulator. The dashed line marks the preregistered resilience threshold (\(3\sigma = 49.7\)).}
    \label{fig:defense_da}
    \vspace{-0.8em}
\end{figure}

\paragraph{Defense Attempts.}
The epoch-level auditor separates Byzantine from honest clients by direction scores (\(0.586\) vs.\ \(0.215\), \(2.7\times\)), yet rejects only about one-third of Byzantine updates per round (\(31.5\%\)\text{--}\(33.0\%\)), leaving sufficient accepted updates to sustain degradation. It reduces \(D\) to \(49.44\) on \(d_0\) and \(89.61\) on \(d_2\), but reaches \(153.59\) on \(d_1\); the latter nominally exceeds the no-defense anchor (\(137.48\)), although the signed comparison is confounded by cross-run variance of the corresponding anchor. Its median degradation of \(89.61\) therefore exceeds the \(3\sigma\) criterion. The temporal accumulator further reduces median degradation to \(38.20\) (with the negative $d_0$ value reflecting cross-run anchor variance), but its maximum degradation remains \(82.26\) on \(d_2\), violating the preregistered dual condition that both median and maximum remain below \(49.7\); it also fails the clean-training gate with relative return \(0.886 < 0.9\).

\paragraph{Defense Outlook.}
These results suggest that parameter-space temporal consistency alone is insufficient when malicious updates are locally bounded yet persistently targeted. The accumulator can substantially increase detection coverage, separating Byzantine signatures by \(5.5\times\) and detecting \(97.5\%\)\text{--}\(98.5\%\) through union rejection, but this comes with \(22.7\%\)\text{--}\(31.8\%\) false-positive rejection on honest updates and still leaves residual attack influence. This detection--suppression gap motivates a more direct defense direction: complement update-space auditing with rollout-level behavioral monitoring of state-conditioned intervention patterns, repeated-state policy changes, and cross-round shifts in state occupancy. Such defenses would target the trajectory-level footprint exploited by \textsc{V-BSA}, rather than relying solely on parameter-space persistence.

\section{Conclusion}
\label{sec:conclusion}

We formalize Byzantine FRL by decoupling intervention timing (\textsc{when}) from update construction (\textsc{how}). Controlled comparisons demonstrate that varying selected trajectory states under fixed \textsc{how} substantially alters policy degradation, establishing intervention timing as an orthogonal attack dimension. Leveraging this insight, \textsc{V-BSA} uses local policy uncertainty to concentrate behavioral steering onto sparse, state-selective interventions, bypassing single-round robust filters and unweighted majority voting ensembles. While effective in discrete control, exploratory evaluations identify operational boundaries in continuous settings due to variance-head saturation (Appendix~\ref{app:continuous_study}). Crucially, our defense audits show that cumulative temporal auditing and robust intra-group estimators maintain substantial resilience, motivating future trajectory-level monitoring of sequential state visitation alongside parameter-space checks to secure federated sequential decision-making.

\bibliography{iclr2027_conference}
\bibliographystyle{iclr2027_conference}

\newpage
\appendix
\section{Theoretical Proofs and Analysis}
\label{app:theoretical_proofs}

This appendix provides the complete derivations underlying the WHEN-side analysis in Section~\ref{sec:when_theory}. We first establish the local Fisher characterization of behavioral susceptibility, then derive the idealized budgeted selection rule, the entropy--susceptibility ranking result for binary-action policies, and the failure of exact ranking equivalence in multi-action settings. We finally provide a separate analytical explanation for the special target-assignment behavior on CartPole. The latter is intended as an analysis of an observed design boundary rather than as a second main theoretical contribution.

\subsection{Local KL Expansion and Fisher Decomposition}
\label{app:proof_kl_expansion}

For a fixed state \(s\), define
\begin{equation}
  f(\delta\theta) \triangleq \mathbb{D}_{\mathrm{KL}}\left(\pi_\theta(\cdot \mid s) \,\parallel\, \pi_{\theta+\delta\theta}(\cdot \mid s)\right).
\end{equation}
Because the action space is finite and the categorical policy has strictly positive action probabilities in the interior of the softmax simplex,
\begin{equation}
  f(\delta\theta) = \sum_{a \in \mathcal{A}} \pi_\theta(a \mid s) \log \pi_\theta(a \mid s) - \sum_{a \in \mathcal{A}} \pi_\theta(a \mid s) \log \pi_{\theta+\delta\theta}(a \mid s).
\end{equation}
Clearly, \(f(\mathbf{0}) = 0\). Its gradient with respect to \(\delta\theta\) is
\begin{equation}
  \nabla_{\delta\theta} f(\delta\theta) = -\sum_{a \in \mathcal{A}} \pi_\theta(a \mid s) \nabla_\theta \log \pi_{\theta+\delta\theta}(a \mid s).
\end{equation}
At \(\delta\theta = \mathbf{0}\),
\begin{equation}
  \nabla_{\delta\theta} f(\mathbf{0}) = -\mathbb{E}_{a \sim \pi_\theta(\cdot \mid s)} \left[ \nabla_\theta \log \pi_\theta(a \mid s) \right] = \mathbf{0},
\end{equation}
by the standard score-function identity. Differentiating a second time yields the Hessian at \(\delta\theta = \mathbf{0}\):
\begin{align}
  \nabla_{\delta\theta}^2 f(\mathbf{0})
  &= -\sum_{a \in \mathcal{A}} \pi_\theta(a \mid s) \nabla_\theta^2 \log \pi_\theta(a \mid s) \nonumber \\
  &= -\sum_{a \in \mathcal{A}} \pi_\theta(a \mid s) \left( \frac{\nabla_\theta^2 \pi_\theta(a \mid s)}{\pi_\theta(a \mid s)} - \frac{\nabla_\theta \pi_\theta(a \mid s) \nabla_\theta \pi_\theta(a \mid s)^\top}{\pi_\theta(a \mid s)^2} \right) \nonumber \\
  &= -\nabla_\theta^2 \left( \sum_{a \in \mathcal{A}} \pi_\theta(a \mid s) \right) + \sum_{a \in \mathcal{A}} \pi_\theta(a \mid s) \nabla_\theta \log \pi_\theta(a \mid s) \nabla_\theta \log \pi_\theta(a \mid s)^\top \nonumber \\
  &= \mathbb{E}_{a \sim \pi_\theta(\cdot \mid s)} \left[ \nabla_\theta \log \pi_\theta(a \mid s) \nabla_\theta \log \pi_\theta(a \mid s)^\top \right] \triangleq F(s),
\end{align}
where \(F(s)\) is the state-conditional Fisher information matrix. Assuming that \(\pi_\theta(\cdot \mid s)\) is \(C^3\) in \(\theta\) with locally bounded third derivatives in a neighborhood of \(\theta\), Taylor's theorem gives
\begin{equation}
  \mathbb{D}_{\mathrm{KL}}\left(\pi_\theta(\cdot \mid s) \,\parallel\, \pi_{\theta+\delta\theta}(\cdot \mid s)\right) = \frac{1}{2} \delta\theta^\top F(s) \delta\theta + \mathcal{O}(\|\delta\theta\|_2^3).
  \label{app_eq:kl_expansion}
\end{equation}

\paragraph{Fisher decomposition for softmax policies.}
Let \(z_\theta(s) \in \mathbb{R}^{|\mathcal{A}|}\) denote the logit vector, \(p = \pi_\theta(\cdot \mid s)\), and \(J_z(s) \triangleq \nabla_\theta z_\theta(s) \in \mathbb{R}^{|\mathcal{A}| \times d}\). For a categorical softmax policy,
\begin{equation}
  \nabla_z \log \pi_\theta(a \mid s) = e_a - p,
\end{equation}
where \(e_a\) is the canonical basis vector corresponding to action \(a\). By the chain rule, \(\nabla_\theta \log \pi_\theta(a \mid s) = J_z(s)^\top (e_a - p)\). Therefore,
\begin{align}
  F(s) &= \sum_{a \in \mathcal{A}} p_a J_z(s)^\top (e_a - p)(e_a - p)^\top J_z(s) \nonumber \\
  &= J_z(s)^\top \left( \operatorname{Diag}(p) - pp^\top \right) J_z(s).
\end{align}
Defining \(M(p) \triangleq \operatorname{Diag}(p) - pp^\top\), we obtain
\begin{equation}
  F(s) = J_z(s)^\top M(p) J_z(s).
  \label{app_eq:fisher_decomposition}
\end{equation}

\subsection{Susceptibility Surrogates}
\label{app:proof_susceptibility}

From Eq.~\eqref{app_eq:kl_expansion}, the local worst-case behavioral susceptibility under an \(\ell_2\) perturbation budget \(\epsilon\) is
\begin{equation}
  \chi_\epsilon(s) \triangleq \sup_{\|\delta\theta\|_2 \le \epsilon} \mathbb{D}_{\mathrm{KL}}\left(\pi_\theta(\cdot \mid s) \,\parallel\, \pi_{\theta+\delta\theta}(\cdot \mid s)\right) = \frac{\epsilon^2}{2} \lambda_{\max}(F(s)) + \mathcal{O}(\epsilon^3).
  \label{app_eq:worst_case_susceptibility}
\end{equation}
This follows because \(F(s) \succeq 0\) and the maximum of the quadratic form \(\delta\theta^\top F(s) \delta\theta\) over the Euclidean ball is attained along an eigenvector associated with \(\lambda_{\max}(F(s))\).

Direct computation of \(\lambda_{\max}(F(s))\) or \(\operatorname{Tr}(F(s))\) requires access to parameter-space sensitivity. For a direction-agnostic second-order surrogate, let \(\delta\theta \sim \mathcal{N}(\mathbf{0}, \sigma^2 I_d)\). The quadratic term in the expected local KL divergence becomes
\begin{equation}
  \mathbb{E}_{\delta\theta} \left[ \frac{1}{2} \delta\theta^\top F(s) \delta\theta \right] = \frac{1}{2} \operatorname{Tr}\left( F(s) \mathbb{E}[\delta\theta \delta\theta^\top] \right) = \frac{\sigma^2}{2} \operatorname{Tr}(F(s)).
\end{equation}
Accordingly, we define
\begin{equation}
  \bar{\chi}^{(2)}_\sigma(s) \triangleq \frac{\sigma^2}{2} \operatorname{Tr}(F(s)).
  \label{app_eq:trace_surrogate}
\end{equation}
Under the local smoothness assumptions of Eq.~\eqref{app_eq:kl_expansion}, the expected KL divergence differs from Eq.~\eqref{app_eq:trace_surrogate} only through higher-order terms in the perturbation scale.

The two quantities in Eqs.~\eqref{app_eq:worst_case_susceptibility} and \eqref{app_eq:trace_surrogate} serve distinct roles: \(\lambda_{\max}(F(s))\) describes the single most susceptible parameter direction, whereas \(\operatorname{Tr}(F(s))\) aggregates sensitivity over an isotropic ensemble of directions. Because the attacker cannot anticipate which parameter directions will be emphasized by an unknown aggregation rule, we use the latter as a direction-agnostic local surrogate.

\subsection{Idealized Budgeted Selection}
\label{app:proof_budget_selection}

Consider a rollout \(\tau = (s_1, \dots, s_K)\) and a hypothetical intra-round intervention budget \(B \le K\). An oracle attacker that seeks to maximize the additive local susceptibility surrogate solves
\begin{equation}
  \max_{g \in \{0, 1\}^K} \sum_{k=1}^K g_k \bar{\chi}^{(2)}_\sigma(s_k) \quad \text{s.t.} \quad \sum_{k=1}^K g_k \le B.
  \label{app_eq:budget_oracle}
\end{equation}

\begin{proposition}[Budgeted Selection Optimality]
\label{prop:budget_selection}
For the oracle problem in Eq.~\eqref{app_eq:budget_oracle}, an optimal selector consists of the \(B\) states having the largest values of \(\operatorname{Tr}(F(s_k))\).
\end{proposition}

\begin{proof}
Because \(\sigma^2/2\) is common to all states, maximizing \(\bar{\chi}^{(2)}_\sigma(s_k)\) is equivalent to maximizing \(\operatorname{Tr}(F(s_k))\). Suppose a feasible solution selects state \(s_j\) but does not select \(s_i\), where \(\operatorname{Tr}(F(s_i)) > \operatorname{Tr}(F(s_j))\). Replacing \(s_j\) with \(s_i\) preserves the cardinality constraint and strictly increases the objective. Hence, no optimal solution can contain a lower-ranked state while excluding a higher-ranked one. Selecting the top-\(B\) states is therefore optimal.
\end{proof}

This result is an oracle characterization of the desired local ranking under a fixed budget. It does not describe the online procedure used by \textsc{V-BSA}, nor does it imply optimal degradation of long-horizon expected return. The latter depends on state-visitation shifts, environmental transition dynamics, future policy updates, and the interaction between successive interventions.

\subsection{Entropy--Susceptibility Ranking under Centered Isotropy}
\label{app:proof_theorem_1}

To obtain an online state-selection rule without explicitly computing the logit Jacobian, consider the centered logit subspace. Define
\begin{equation}
  P_\perp \triangleq I_{|\mathcal{A}|} - \frac{1}{|\mathcal{A}|} \mathbf{1}\mathbf{1}^\top.
\end{equation}
Suppose
\begin{equation}
  J_z(s) J_z(s)^\top = c P_\perp,
  \label{app_eq:centered_isotropy}
\end{equation}
where \(c > 0\) is a state-independent constant. Since \(M(p)\mathbf{1} = \mathbf{0}\), we have
\begin{equation}
  \operatorname{Tr}(F(s)) = \operatorname{Tr}\left( M(p) J_z(s) J_z(s)^\top \right) = c \operatorname{Tr}\left( M(p) P_\perp \right) = c\left(1 - \|p\|_2^2\right).
  \label{app_eq:fisher_trace_centered}
\end{equation}

\begin{theorem}[Restatement of Theorem~\ref{thm:entropy_susceptibility}]
\label{app_thm:entropy_susceptibility}
For a binary-action policy \(p = (q, 1-q)\) with \(q \in (0, 1)\), if the centered isotropy condition in Eq.~\eqref{app_eq:centered_isotropy} holds with state-independent \(c > 0\), then Shannon entropy \(H(p) = -q \log q - (1-q)\log(1-q)\) and the Fisher-trace surrogate \(\operatorname{Tr}(F(s))\) induce identical state rankings:
\begin{equation}
  H(p_1) > H(p_2) \iff \operatorname{Tr}(F(s_1)) > \operatorname{Tr}(F(s_2)).
\end{equation}
\end{theorem}

\begin{proof}
For \(|\mathcal{A}| = 2\),
\begin{equation}
  P_\perp = I_2 - \frac{1}{2}\mathbf{1}\mathbf{1}^\top = \begin{bmatrix} 1/2 & -1/2 \\ -1/2 & 1/2 \end{bmatrix}.
\end{equation}
Moreover,
\begin{equation}
  M(p) = \begin{bmatrix} q(1-q) & -q(1-q) \\ -q(1-q) & q(1-q) \end{bmatrix} = 2q(1-q) P_\perp.
\end{equation}
Using Eq.~\eqref{app_eq:centered_isotropy},
\begin{equation}
  \operatorname{Tr}(F(s)) = \operatorname{Tr}\left( M(p) J_z(s) J_z(s)^\top \right) = 2c\,q(1-q) \operatorname{Tr}(P_\perp^2) = 2c\,q(1-q),
\end{equation}
because \(P_\perp^2 = P_\perp\) and \(\operatorname{Tr}(P_\perp) = 1\).
Now define \(\psi(q) \triangleq 2c\,q(1-q)\) and \(\phi(q) \triangleq H(q)\). Both functions are symmetric around \(q = 1/2\). On \(q \in (0, 1/2)\),
\begin{align}
  \psi'(q) &= 2c(1 - 2q) > 0, \\
  \phi'(q) &= \log\left(\frac{1-q}{q}\right) > 0.
\end{align}
Hence both functions are strictly increasing on \((0, 1/2]\) and strictly decreasing on \([1/2, 1)\). Equivalently, both are strictly decreasing functions of \(|q - 1/2|\). Therefore,
\begin{align}
  H(p_1) > H(p_2)
  &\iff \left|q_1 - \frac{1}{2}\right| < \left|q_2 - \frac{1}{2}\right| \nonumber \\
  &\iff 2c\,q_1(1-q_1) > 2c\,q_2(1-q_2) \nonumber \\
  &\iff \operatorname{Tr}(F(s_1)) > \operatorname{Tr}(F(s_2)).
\end{align}
\end{proof}

\paragraph{From ranking to online gating.}
Theorem~\ref{thm:entropy_susceptibility} bridges the idealized ranking formulation to an online state-selection rule. Under its binary-action conditions, entropy and the Fisher-trace surrogate induce the identical ordering. Consequently, an entropy level set
\begin{equation}
  \mathcal{S}_h \triangleq \left\{ s \;\Big|\; H(\pi_\theta(\cdot \mid s)) > h \right\}
\end{equation}
corresponds to a top-ranked subset under the susceptibility surrogate, with an induced intervention budget \(B(h) = |\mathcal{S}_h|\). Thus, thresholding entropy realizes the susceptibility ranking without explicitly evaluating \(\operatorname{Tr}(F(s))\), buffering the rollout for sorting, or performing Jacobian backpropagation.

For multi-action policies, exact ranking equivalence does not generally hold. \textsc{V-BSA} therefore uses normalized entropy as a practical output-only proxy:
\begin{equation}
  \mathcal{G}_{\mathrm{when}}(s; \theta^t) = \mathbf{1}\left[ \frac{H(\pi_{\theta^t}(\cdot \mid s))}{\log |\mathcal{A}|} > \tau \right],
  \label{app_eq:entropy_gate}
\end{equation}
where \(\tau = 0.9\) is fixed \emph{a priori} and is neither swept nor calibrated to match a prescribed intervention budget. The resulting intervention frequency is therefore induced endogenously by the evolving policy and trajectory distribution.

\subsection{Analytical Failure of Exact Entropy Ranking in Multi-Action Regimes}
\label{app:proof_multiaction_counterexample}

For \(|\mathcal{A}| \ge 3\), Eq.~\eqref{app_eq:fisher_trace_centered} becomes
\begin{equation}
  \operatorname{Tr}(F(s)) = c\left(1 - \|p\|_2^2\right).
\end{equation}
The quantity \(1 - \|p\|_2^2\) is the Gini impurity (equivalently Tsallis-2 entropy up to normalization), which does not in general induce the same ordering as Shannon entropy.
Consider a four-action policy and two states with
\begin{align}
  p_A &= (0.55, \; 0.20, \; 0.20, \; 0.05), \\
  p_B &= (0.40, \; 0.40, \; 0.19, \; 0.01).
\end{align}
For \(p_A\),
\begin{equation}
  \|p_A\|_2^2 = 0.55^2 + 0.20^2 + 0.20^2 + 0.05^2 = 0.3850 \implies 1 - \|p_A\|_2^2 = 0.6150.
\end{equation}
For \(p_B\),
\begin{equation}
  \|p_B\|_2^2 = 0.40^2 + 0.40^2 + 0.19^2 + 0.01^2 = 0.3562 \implies 1 - \|p_B\|_2^2 = 0.6438.
\end{equation}
Hence, under the Fisher trace surrogate, \(\operatorname{Tr}(F(s_B)) > \operatorname{Tr}(F(s_A))\).
In contrast, the Shannon entropies evaluate to:
\begin{align}
  H(p_A) &= -0.55\ln 0.55 - 2(0.20\ln 0.20) - 0.05\ln 0.05 \approx 1.1224, \\
  H(p_B) &= -2(0.40\ln 0.40) - 0.19\ln 0.19 - 0.01\ln 0.01 \approx 1.0946.
\end{align}
Therefore,
\begin{equation}
  H(p_A) > H(p_B) \quad \text{while} \quad \operatorname{Tr}(F(s_A)) < \operatorname{Tr}(F(s_B)).
\end{equation}
This explicit ranking reversal shows that, for multi-action policies, Shannon entropy cannot in general be interpreted as an exact estimator of Fisher-trace susceptibility. The entropy gate in \textsc{V-BSA} should consequently be understood as a lightweight practical proxy rather than an exact susceptibility oracle.

\subsection{Aggregation Analysis on CartPole: Binary Steering Geometry and Target Cancellation}
\label{app:cartpole_aggregation}

This section analyzes the CartPole-specific behavior of \textsc{V-BSA} through the geometry of binary target steering. It is intended as a mechanistic explanation of the observed aggregation-dependent phenomena, rather than as a second main theoretical contribution. The formal theoretical analysis of the paper remains the WHEN-side susceptibility characterization in Section~\ref{sec:when_theory} and Appendix~\ref{app:theoretical_proofs}.

\paragraph{Setup and target assignment.}
For the signed update convention of Eq.~\eqref{eq:steering_construction}, let
\begin{equation}
  d(a; s) \triangleq -\nabla_\theta \mathcal{L}_{\mathrm{steer}}(s, a) = \nabla_\theta \log \pi_\theta(a \mid s),
\end{equation}
so that the malicious update before the local viability constraint is
\begin{equation}
  v = g + d.
\end{equation}
For CartPole, the main configuration uses the same deterministic modulo target rule as the multi-action environments:
\begin{equation}
  a^{(g)} = g \bmod 2.
\end{equation}
With five compromised groups, this yields the target pattern \(\{0, 1, 0, 1, 0\}\), i.e., \(n_0 = 3\) groups target action $0$ and \(n_1 = 2\) groups target action $1$. Each group contains \(\nu = 2\) compromised clients. The timing gate is
\begin{equation}
  \mathcal{S}_\tau \triangleq \left\{ s \;\Big|\; \frac{H(\pi_\theta(\cdot \mid s))}{\ln 2} > \tau \right\}, \qquad \tau = 0.9.
\end{equation}
The resulting malicious updates are globally norm-matched to the corresponding nominal local updates and then coordinate-wise clamped to a client-local empirical update envelope with \(\lambda_{\mathrm{box}} = 0.85\). The five group policies remain separate throughout training and are combined only at decision time by majority voting.

\paragraph{Binary target geometry.}
Let
\begin{equation}
  \Delta(s) \triangleq \nabla_\theta z_0(s) - \nabla_\theta z_1(s),
\end{equation}
where \(z_0(s)\) and \(z_1(s)\) are the two action logits. For \(\pi_\theta(\cdot \mid s) = (\pi, 1-\pi)\), the score identities for a binary softmax policy give
\begin{equation}
\begin{aligned}
  d(0; s) &= (1 - \pi)\Delta(s), \\
  d(1; s) &= -\pi\Delta(s),
\end{aligned}
\end{equation}
and hence
\begin{equation}
  d(0; s) + d(1; s) = (1 - 2\pi)\Delta(s).
  \label{app_eq:binary_geometry}
\end{equation}
Thus, for every fixed state, the two target-specific steering directions are exactly collinear and oppositely directed. They cancel exactly at the maximum-entropy point \(\pi = 1/2\).

The entropy gate also restricts the policy probability to a neighborhood of \(1/2\). Let \(q_\tau \in (0, 1/2)\) satisfy \(H(q_\tau) = 0.9 \ln 2\). Numerically, \(q_\tau \approx 0.316\). Therefore,
\begin{equation}
  \frac{H(\pi)}{\ln 2} > 0.9 \implies \pi \in (q_\tau, 1 - q_\tau),
\end{equation}
and consequently
\begin{equation}
  |1 - 2\pi| < 1 - 2q_\tau \approx 0.368.
\end{equation}
Within the gate set,
\begin{equation}
  \|d(0; s) + d(1; s)\|_2 < 0.368\,\|\Delta(s)\|_2,
\end{equation}
whereas each individual target direction satisfies
\begin{equation}
  \|d(a; s)\|_2 \ge 0.316\,\|\Delta(s)\|_2, \qquad a \in \{0, 1\}.
\end{equation}
Hence, the entropy gate selects precisely the region in which the two opposing target directions have comparable magnitude, making cancellation a structural possibility for an alternating target assignment.

\paragraph{Cancellation under alternating assignment.}
To isolate the geometric effect, consider the following matched-state abstraction: all Byzantine workers are assumed to act on the same state \(s\), with the same policy probability \(\pi\), the same \(\Delta(s)\), and equal steering strength. Under this abstraction, the aggregate Byzantine steering before subsequent aggregation-specific processing is
\begin{equation}
\begin{aligned}
  A(\pi) &= \nu \left[ n_0 d(0; s) + n_1 d(1; s) \right] \\
  &= \nu \left[ n_0 - (n_0 + n_1)\pi \right] \Delta(s).
  \label{app_eq:binary_aggregate}
\end{aligned}
\end{equation}
At the maximum-entropy center \(\pi = 1/2\):
\begin{equation}
  A\left(\frac{1}{2}\right) = \frac{\nu}{2}(n_0 - n_1)\Delta(s).
\end{equation}
For the realized assignment \(n_0 = 3\), \(n_1 = 2\), and \(\nu = 2\):
\begin{equation}
  A\left(\frac{1}{2}\right) = \Delta(s).
\end{equation}
For comparison, if all \(G = 5\) compromised groups were aligned to action $0$, the corresponding center-point reference would be
\begin{equation}
  10\, d(0; s) = 5\Delta(s).
\end{equation}
Thus, the alternating assignment leaves only one fifth of the fully aligned reference magnitude at \(\pi = 1/2\). More generally, when the number of target-0 and target-1 groups is equal, the matched-state aggregate vanishes exactly at the maximum-entropy point.

Equation~\eqref{app_eq:binary_aggregate} is a pointwise geometric characterization rather than an identity for the realized federated trajectory. In the actual training process, compromised clients encounter different states and accumulate updates over their own trajectories, so the exact matched-state assumptions need not hold globally. The equation therefore identifies a cancellation tendency induced by the binary target geometry; its quantitative effect depends on the downstream aggregation mechanism.

\paragraph{Relation to the empirical CartPole behavior.}
This geometric interpretation is consistent with the observed dependence on the aggregation surface. Under the original modulo assignment, the in-group coordinate-wise trimmed-mean and median blocks do not meet the pre-registered damage criterion on the tested CartPole configurations. By contrast, the FedPG-BR block, whose final decision mechanism operates through the separately maintained group policies and majority voting rather than a direct in-group order statistic, exhibits seed-dependent damage; on the three dynamics seeds, the corresponding values are \(D = \{164.9, 0, 0\}\). This difference indicates that the same local steering geometry can produce different task-level outcomes depending on how the resulting local updates are processed downstream.

A pre-registered CartPole variant further tested the effect of removing the alternating assignment by assigning all five compromised groups to action $0$. This fixed-target variant passed its pre-registered damage criterion on \(2/3\) dynamics, yielding \(D = \{139.5, 0, 192.4\}\), but it did not amplify the largest observed single-seed damage relative to the original modulo configuration, since \(139.5 < 164.9\). Accordingly, the fixed-target result is interpreted as evidence that target alignment changes the success pattern in the tested setting, not as a uniformly stronger attack. Both target-assignment rules were fixed \emph{a priori}; neither result is used to retune the attack.

\paragraph{Local update envelope.}
The local constraint has a narrower interpretation than a statistical indistinguishability or defense-acceptance guarantee. Because the malicious update is norm-matched and then coordinate-wise clamped inside the client-local empirical envelope, each coordinate remains within the prescribed empirical range formed from that client's own update history. This construction bounds the update relative to the attacker's local historical scale, but does not imply that the resulting update is indistinguishable from honest updates or that any particular aggregation rule must accept it.

\paragraph{Scope of the analysis.}
The conclusions above are intentionally limited to the tested setting. First, the cancellation expression in Eq.~\eqref{app_eq:binary_aggregate} relies on a matched-state abstraction and does not constitute an aggregation-independent impossibility result. Second, the observed effect is aggregation-dependent: the analysis directly describes the binary steering geometry and its interaction with the tested in-group order-statistic aggregators, while FedPG-BR and the single-pool setting involve additional downstream mechanisms. Third, the analysis does not imply a converse rule that action spaces with \(|\mathcal{A}| \ge 3\) must admit successful penetration. The experiments instead reveal a task- and aggregation-dependent boundary, with the empirical penetration pattern expanding from the binary-action setting to the three- and four-action settings under the evaluated configurations.


\section{Behavioral Susceptibility Metric}
\label{app:when_metric}

In Section~\ref{sec:exp_results}, we conduct a mechanistic audit of the \textsc{when} selection mechanism by evaluating the behavioral susceptibility of visited trajectory states. This section formalizes the perturbation-based behavioral susceptibility metric \(L(s)\), describes its empirical Monte-Carlo estimator, and details the policy-entropy surrogate evaluated during online execution.

\paragraph{Definition of Behavioral Susceptibility.}
Let \(\pi_\theta(\cdot \mid s)\) denote the evaluated policy at state \(s \in \mathcal{S}\), where the action space \(\mathcal{A}\) is discrete. We measure local behavioral sensitivity by the total variation distance between the nominal action distribution and the distribution induced by isotropic state perturbations:
\begin{equation}
L(s) \triangleq \frac{1}{2} \, \mathbb{E}_{\varepsilon \sim \mathcal{N}(0,\,\sigma^2 I_{d_s})} \left[ \big\|\pi_\theta(\cdot \mid s + \varepsilon) - \pi_\theta(\cdot \mid s)\big\|_1 \right] \in [0, 1],
\label{eq:susceptibility}
\end{equation}
where \(d_s\) denotes the state dimension and \(\sigma\) is scaled relative to the empirical dispersion of the state variables. Larger \(L(s)\) indicates that small input perturbations induce larger shifts in the action distribution, corresponding to higher local behavioral susceptibility, whereas smaller values characterize locally stable decision regimes.

\paragraph{Empirical Monte-Carlo Estimator.}
During offline auditing, \(\hat{L}(s)\) is evaluated without model updates or backpropagation via \(M = 50\) independent Monte-Carlo noise samples \(\{\varepsilon_m\}_{m=1}^M \sim \mathcal{N}(0, \sigma^2 I_{d_s})\):
\begin{equation}
\hat{L}(s) = \frac{1}{2M} \sum_{m=1}^{M} \big\|\pi_\theta(\cdot \mid s + \varepsilon_m) - \pi_\theta(\cdot \mid s)\big\|_1.
\label{eq:susceptibility_est}
\end{equation}

\paragraph{Entropy as a Low-Overhead Computational Proxy.}
While the Monte-Carlo estimator \(\hat{L}(s)\) is suitable for post-hoc auditing, evaluating multiple perturbed states at every rollout step would introduce additional forward-pass latency. Therefore, \textsc{V-BSA} uses the normalized policy entropy
\begin{equation}
h(s) = \frac{H(\pi_\theta(\cdot \mid s))}{\log |\mathcal{A}|}
\end{equation}
as a low-overhead online proxy for state selection (Eq.~\eqref{eq:entropy_gate}). Under the centered-isotropy condition of Theorem~\ref{thm:entropy_susceptibility}, policy entropy induces the same susceptibility ranking as the Fisher-trace criterion for binary action spaces. For multi-action tasks, entropy serves as a practical output-uncertainty proxy rather than an exact susceptibility measure. In the LunarLander experiments, this gate induces an average empirical intervention rate of \(\rho \approx 5.7\%\).

\section{System Details, Benchmark Environments, and Baseline Aggregations}
\label{app:system_details}

This section specifies the benchmark control environments, evaluation dynamics, server-side aggregation baselines, and implementation hyperparameters used throughout the paper.

\subsection{Benchmark Environments and Control Dynamics}
\label{app:environments}

Table~\ref{tab:environments} summarizes the structural specifications of the benchmark environments.

\begin{table}[h]
\centering
\small
\renewcommand{\arraystretch}{1.15}
\begin{tabular}{lcccc}
\toprule
\textbf{Environment} & \textbf{State Dim.} \(d_s\) & \textbf{Action Space} \(\mathcal{A}\) & \textbf{Episode Horizon} & \textbf{Task Criterion} \\
\midrule
LunarLander-v2 & 8 & Discrete(4) & 1000 & Return \(\ge 200\) \\
CartPole-v1 & 4 & Discrete(2) & 500 & Return \(= 500\) \\
Acrobot-v1 & 6 & Discrete(3) & 500 & Return \(\ge -100\) \\
LunarLanderContinuous-v2 & 8 & Box(2) & 1000 & Return \(\ge 200\) \\
\bottomrule
\end{tabular}
\caption{\textbf{Environment specifications.} Benchmark tasks covering binary-action, multi-action, and continuous-control settings.}
\label{tab:environments}
\end{table}

\paragraph{CartPole-v1 (\(|\mathcal{A}|=2\)) \citep{barto1983neuronlike, brockman2016openai}.}
CartPole models an unactuated pole hinged to a cart moving along a frictionless horizontal track. The state is four-dimensional, \(s = (x, \dot{x}, \theta, \dot{\theta})\), comprising cart position, cart velocity, pole angle, and pole angular velocity. The discrete action space is \(\mathcal{A} = \{0, 1\}\), corresponding to applying a constant horizontal force of \(10\,\mathrm{N}\) to the left or right, respectively. The system awards a reward of \(+1\) per survival step. An episode terminates when \(|x| > 2.4\), \(|\theta| > 12^\circ\), or the maximum horizon limit \(H = 500\) is reached, yielding a maximum nominal return of \(500\).

\paragraph{Acrobot-v1 (\(|\mathcal{A}|=3\)) \citep{sutton1995generalization, brockman2016openai}.}
Acrobot consists of two serial links operating in a vertical plane, with torque actuation restricted strictly to the second joint. The observation space is six-dimensional, \(s = (\cos\theta_1, \sin\theta_1, \cos\theta_2, \sin\theta_2, \dot{\theta}_1, \dot{\theta}_2)\), and the action space is \(\mathcal{A} = \{0, 1, 2\}\), providing actuated torques of \(-1\), \(0\), or \(+1\,\mathrm{N\cdot m}\). The objective is to swing the lower link tip above a designated clearance line:
\begin{equation}
  -\cos(\theta_1) - \cos(\theta_1 + \theta_2) > 1.0.
\end{equation}
Each non-terminal step incurs a penalty of \(-1\), while achieving target clearance terminates the episode with reward \(0\). Episode returns are bounded within \([-500, 0]\) under a horizon limit of \(H = 500\).

\paragraph{LunarLander-v2 (\(|\mathcal{A}|=4\)) \citep{brockman2016openai}.}
LunarLander simulates the trajectory guidance of a landing module descending toward a designated landing pad. The state space is eight-dimensional, \(s = (x, y, \dot{x}, \dot{y}, \theta, \dot{\theta}, c_1, c_2)\), where \(c_1, c_2 \in \{0, 1\}\) denote ground contact indicators for the two landing legs. The action space is discrete, \(\mathcal{A} = \{0, 1, 2, 3\}\), corresponding to firing no engine, the left orientation booster, the main thruster, or the right orientation booster. The reward function balances distance-to-pad and velocity shaping with engine firing costs; successful touchdown awards \(+100\), whereas crashing incurs a \(-100\) penalty. Episodes terminate upon successful landing, module destruction, exceeding horizontal limits (\(|x| \ge 1.0\)), or reaching \(H = 1000\).

\paragraph{Continuous LunarLander (LunarLanderContinuous-v2) \citep{brockman2016openai}.}
The exploratory continuous-control setting shares identical underlying dynamics and state variables with the discrete task, but defines an action space \(\mathcal{A} = [-1, 1]^2\). The coordinate \(a_1 \in [-1, 1]\) governs main thruster throttle: values in \([-1, 0]\) disable thrust, while \((0, 1]\) scales thrust linearly from \(50\%\) to \(100\%\). The lateral coordinate \(a_2 \in [-1, 1]\) controls side thrusters: values in \((-0.5, 0.5)\) turn boosters off, while \(a_2 \le -0.5\) and \(a_2 \ge 0.5\) activate the left and right thrusters, respectively. This setting is examined in Appendix~\ref{app:continuous_study} to evaluate the operational boundaries of entropy gating under continuous Gaussian policies.

\paragraph{Evaluation Dynamics Protocol (\(d_0, d_1, d_2\)).}
To evaluate system resilience under environmental variation, all benchmark evaluations are conducted across three fixed dynamic realizations, denoted \(d_0, d_1, d_2\). The same dynamic indices are assigned identically between nominal clean baselines and attacked runs, allowing degradation \(D\) to be reported per condition alongside summary statistics.

\subsection{Server-Side Robust Aggregation Rules}
\label{app:aggregation_rules}

Let \(\{u_i^t\}_{i=1}^{N}\) denote the update vectors submitted by \(N\) clients at communication round \(t\), where \(u_i^t \in \mathbb{R}^d\). For standard aggregation rules, these represent local parameter updates; for FedPG-BR, they correspond to local policy-gradient estimates. An aggregation rule \(\operatorname{Agg}(\cdot)\) produces the global update vector \(u_{\mathrm{global}}^t\).

\paragraph{Federated Averaging (FedAvg) \citep{mcmahan2017communication}.}
FedAvg computes the coordinate-wise arithmetic mean of all submitted vectors:
\begin{equation}
  u_{\mathrm{global}}^t = \frac{1}{N} \sum_{i=1}^{N} u_i^t.
\end{equation}
No Byzantine filtering is applied.

\paragraph{Coordinate-Wise Trimmed Mean \citep{yin2018byzantine}.}
For each coordinate \(j \in \{1, \dots, d\}\), the server sorts \(\{u_{1,j}^t, \dots, u_{N,j}^t\}\) in non-decreasing order, discards the \(c\) lowest and \(c\) highest values, and computes the mean of the remaining \(N - 2c\) elements:
\begin{equation}
  [u_{\mathrm{global}}^t]_j = \frac{1}{N - 2c} \sum_{r = c + 1}^{N - c} u_{(r), j}^t,
\end{equation}
where \(u_{(r), j}^t\) denotes the \(r\)-th order statistic for coordinate \(j\). In ensemble evaluations (6 clients per group), we set \(c = 1\); in non-ensemble evaluations (\(N=30\)), we set \(c = 5\).

\paragraph{Coordinate-Wise Median \citep{yin2018byzantine}.}
The server computes the median along each parameter coordinate independently:
\begin{equation}
  [u_{\mathrm{global}}^t]_j = \operatorname{median}\left( \{u_{i, j}^t\}_{i=1}^{N} \right).
\end{equation}
For an even number of inputs, the median is taken as the mean of the two central values.

\paragraph{FedPG-BR \citep{fan2021fault}.}
FedPG-BR filters submitted policy-gradient vectors using directional and magnitude alignment against a robust reference. The server determines candidate coordinate medians and selects a reference update vector. Gradients whose Euclidean distance from this reference exceeds a threshold are filtered out, and the remaining updates are averaged. In our implementation, FedPG-BR is applied to submitted policy gradients within each evaluated client pool.

\paragraph{Ensemble Subgroup Voting \citep{fang2025provably}.}
In ensemble configurations, the \(N = 30\) clients are partitioned into \(K = 5\) disjoint subgroups, \(\mathcal{G}_1, \dots, \mathcal{G}_5\), with 6 clients per group (including 2 compromised clients per group). Each subgroup independently executes an aggregation rule to optimize its policy network \(\pi_{\theta_k}\). At test time, subgroup policies evaluate environment states independently, and the executed action is determined via majority voting:
\begin{equation}
  a^*(s) = \arg\max_{a \in \mathcal{A}} \sum_{k=1}^{K} \mathbf{1}\left[ a = \arg\max_{a' \in \mathcal{A}} \pi_{\theta_k}(a' \mid s) \right].
\end{equation}
This architecture enforces decision-level redundancy across subgroup policies rather than parameter-space consensus across the full client population.

\begin{table}[h]
\centering
\small
\renewcommand{\arraystretch}{1.18}
\begin{tabular}{ll}
\toprule
\textbf{Component} & \textbf{Specification / Value} \\
\midrule
\multicolumn{2}{l}{\textbf{Policy Architecture \& Optimization}} \\
Network Architecture & MLP: 2 hidden layers, 64 units each, Tanh activations \\
Output Layer & Discrete: categorical logits (\(\mathbb{R}^{|\mathcal{A}|}\)); Continuous: Gaussian mean and log-variance \\
Optimizer & Adam (\(\beta_1 = 0.9, \beta_2 = 0.999\)), learning rate \(\eta = 10^{-3}\) \\
Discount Factor \(\gamma\) & \(0.99\) \\
Rollout Budget & Mini-batch size \(b = 8\); inner steps sampled via geometric schedule (\(B=32\)) \\
Episode Horizon & Environment-specific (Table~\ref{tab:environments}) \\
\midrule
\multicolumn{2}{l}{\textbf{Federated Consensus and Robust Aggregation}} \\
Federation Scale & \(N = 30\) total clients, \(M = 10\) Byzantine clients (\(\beta = 1/3\)) \\
Ensemble Configuration & \(K = 5\) groups (6 clients per group; 2 Byzantine clients per group) \\
FedPG-BR Parameters & Reference variance \(\sigma_{\mathrm{ref}} = 0.07\), threshold scales \(\hat{\lambda} = 1.0, \hat{\zeta} = 0.02\) \\
Coordinate-wise Median & Coordinate-wise median across client gradient vectors \\
Trimmed Mean & Coordinate-wise trimming with \(c = 1\) (ensemble) or \(c = 5\) (non-ensemble) \\
Total Training Rounds & \(T = 261\) global communication rounds \\
\midrule
\multicolumn{2}{l}{\textbf{\textsc{V-BSA} Attack Parameters}} \\
State Selection Gate & Policy entropy threshold \(\tau = 0.9\) (normalized by \(\log |\mathcal{A}|\)) \\
Behavioral Steering & Steering loss \(\mathcal{L}_{\mathrm{steer}}\), unit scaling weight, target derived via \(\operatorname{group\_id} \pmod{|\mathcal{A}|}\) \\
Norm Matching & \(\ell_2\)-norm aligned to benign local gradient \(\|g_i^t\|_2\) \\
Empirical Envelope & EMA smoothing coefficient \(\alpha_{\mathrm{ema}} = 0.1\), coordinate clipping scale \(\lambda_{\mathrm{box}} = 0.85\) \\
\midrule
\multicolumn{2}{l}{\textbf{Evaluation Protocol}} \\
Terminal Return \(J\) & Evaluated over 10 independent test rollouts at round \(T=261\) \\
Degradation Metric & \(D \triangleq J_{\mathrm{clean}} - J_{\mathrm{attacked}}\) \\
\bottomrule
\end{tabular}
\caption{\textbf{Hyperparameter configuration.} Fixed parameters applied across federated reinforcement learning experiments.}
\label{tab:hparams}
\end{table}

\subsection{Federated Learning and Training Hyperparameters}
\label{app:environments:hparams}

\paragraph{Policy Network Architecture and Scope.}
The benchmark control environments operate over low-dimensional physical state vectors (\(d_s \le 8\)) rather than spatial image grids, for which multi-layer perceptrons (MLPs) constitute the standard, task-appropriate policy representation. Following established practice in Byzantine-robust federated reinforcement learning \citep{fan2021fault,fang2025provably}, we standardize the policy architecture across all discrete and continuous benchmarks to a two-layer MLP with 64 hidden units per layer and Tanh activations. Prior Byzantine FRL literature typically fixes a task-appropriate policy architecture to isolate the interaction between adversarial updates and server-side aggregation rules, avoiding network architecture as a confounding experimental factor. Formally, the operators of \textsc{V-BSA} interface with any parameterized differentiable policy \(\pi_\theta\): the \textsc{when} gate evaluates output-space dispersion via policy entropy, behavioral steering evaluates standard policy gradients \(-\nabla_\theta \mathcal{L}_{\mathrm{steer}}\), and the empirical envelope constrains coordinate-wise deviations in parameter space. We report our empirical findings under this standard architecture as a controlled experimental baseline; evaluating the interaction between state-selective timing and alternative policy parameterizations remains an avenue for future work.

Table~\ref{tab:hparams} details the complete hyperparameters utilized across nominal training, Byzantine aggregation, and the \textsc{V-BSA} attack.

\subsection{Compute Environment}
\label{app:environments:compute}

All experiments were implemented in Python using PyTorch and executed on Linux compute nodes equipped with NVIDIA RTX 5090 GPUs or 48GB GPU instances. Each training configuration was run as a single federated-learning process, with the 30 simulated clients executing their local rollouts and policy updates sequentially within each communication round. A complete 261-round training run typically required approximately 1.5--3.5 hours, depending on the control environment and rollout length.

\section{Complete Training Dynamics Across Aggregations}
\label{app:training_dynamics}

This section provides the complete empirical results underlying the terminal degradation comparisons reported in Table~\ref{tab:main_results}. We report detailed numerical breakdowns, timing schedule ablations, and full convergence trajectories across tasks and aggregations.

\subsection{Detailed Numerical Breakdown and Dynamics Analysis}
\label{app:dynamics:ll_breakdown}

Tables~\ref{tab:ll_ens} and~\ref{tab:ll_nbe} provide the numerical breakdown of clean terminal returns (\(J_{\mathrm{ctl}}\)), attacked terminal returns (\(J_{\mathrm{atk}}\)), and resulting degradation (\(D\)) across the three evaluation dynamics for LunarLander-v2. Beyond aggregate statistics, we examine the strongest ensemble configuration—LunarLander with FedPG-BR—to illustrate how the observed degradation arises at the level of the federated sequential decision process.

\paragraph{Why Ensemble FedPG-BR Resists Conventional Update-Space Attacks.}
A representative benchmark is provided by the Byzantine-robust ensemble defense studied by Fang et al.~\citep{fang2025provably}. Their Normalized attack explicitly maximizes the angular deviation between benign and poisoned updates in parameter space to disrupt policy learning. In their defense architecture, clients are partitioned into independent subgroups (\(K=5\)), each learning an isolated policy via FedPG-BR, with final decisions resolved via test-time majority voting. This design establishes a dual barrier against update-space poisoning:
\begin{enumerate}[label=(\roman*), leftmargin=1.5em, itemsep=1pt]
\item \emph{Intra-group filtering:} Update-space attacks distribute parameter perturbations across all state-action pairs indiscriminately, yielding gradient estimates whose global Euclidean or angular deviation is readily excised by FedPG-BR's reference-distance filter within each subgroup;
\item \emph{Decision-level consensus:} Even if conservative parameter shifts evade local filtering, perturbing the parameters of a minority of subgroups fails to alter the final majority vote, as uncompromised subgroup policies maintain consensus on routine states.
\end{enumerate}

\paragraph{Penetration via Trajectory-Level Behavioral Steering.}
The failure of this dual defense against \textsc{V-BSA} stems from the sequential nature of reinforcement learning. With five groups and two Byzantine clients per group, \textsc{V-BSA} operates inside every group under a defense-agnostic boundary. Rather than attempting to override the server aggregation or maximize parameter divergence, \textsc{V-BSA} applies targeted behavioral steering strictly to a sparse subset of decision states (\(\rho \approx 5.7\%\)) selected via policy entropy. 

Because routine states remain unperturbed, the resulting malicious update closely aligns with nominal gradient scales and satisfies the empirical coordinate envelope \(\mathcal{B}_i^t\), successfully passing FedPG-BR filtering alongside benign peers within each subgroup. Once accepted, the perturbation does not terminate at parameter aggregation; instead, it alters subsequent state-action visitation throughout future local rollouts. Thus, the relevant causal propagation path shifts from single-step parameter distortion to sequential behavioral coupling:
\begin{align*}
\text{\emph{Update-space:}} \quad & \text{malicious update} \;\longrightarrow\; \text{aggregated parameter}, \\
\text{\emph{Trajectory-level:}} \quad & \text{selected states} \;\longrightarrow\; \text{behavioral steering} \;\longrightarrow\; \text{local update} \\
& \quad \;\longrightarrow\; \text{subgroup policy} \;\longrightarrow\; \text{trajectory drift}.
\end{align*}
Because a perturbed action at a critical trajectory bifurcation shifts the subsequent state trajectory, the empirical impact accumulates across consecutive communication rounds rather than remaining confined to a single parameter update.

\paragraph{Empirical Realization and Microscopic Telemetry on LunarLander.}
This trajectory-level mechanism translates into severe terminal disruption under ensemble FedPG-BR. Across dynamics \(d_0, d_1, d_2\), clean baseline returns of \((228.6, 211.9, 245.4)\) degrade to \((80.5, 74.4, -125.2)\) under attack (\(D = \{148.17, 137.48, 370.59\}\), \(\bar{D} = 218.75\)), culminating in complete mission failure on \(d_2\) (\(J_{\mathrm{atk}} = -125.2\), \(D = 370.59\)). Step-level telemetry logged at round \(T=261\) on \(d_2\) directly reveals how this degradation arises from the progressive dismantling of decision-level consensus. Whereas nominal evaluation rollouts (\(\bar{J}_{\mathrm{ctl}} = 242.50\)) rarely deadlock—producing top-two ties in only \(7.9\%\) of decisions (333/4200 steps) with splits typically resolved by decisive majorities—\textsc{V-BSA} drives split votes to \(98.2\%\) (1818/1851 steps) and deadlocked top-two ties to \(38.0\%\) (704/1851 steps), collapsing mean returns to \(-66.71\) and trajectory lengths from 420.0 to 185.1 steps.

A concrete kinematic bifurcation is illustrated in Episode 2 (clean return \(+213.42\) vs.\ attacked return \(-207.28\)). At step 27, the descending module enters a critical hover-adjustment state:
\begin{equation*}
\begin{aligned}
s_{27} = [ & x = -0.029, \; y = 1.295, \; v_x = -0.109, \; v_y = -0.538, \\
           & \theta = 0.033, \; \dot{\theta} = 0.024, \; c_1 = 0, \; c_2 = 0 ].
\end{aligned}
\end{equation*}
At this state, Group 1's policy entropy spikes to \(H = 1.282\), exceeding the unnormalized selection threshold (\(0.9\ln 4 \approx 1.2477\)), while the remaining four groups remain far below the gate (\(\max H \le 0.158\)). In test-time majority voting, the five groups split their votes as \([0, 1, 2, 2, 0]\), producing a tie between main thruster activation (action 2) and engine idling (action 0) that resolves to idling. Failing to fire the main thruster while descending at \(v_y = -0.538\,\mathrm{m/s}\) forfeits the critical deceleration window and accelerates downward descent. Subsequent attitude corrections collapse: by step 228, the lander over-rotates to \(\theta = 0.63\,\mathrm{rad}\) (\(>36^\circ\)) with a single dragged leg (\(c_1=1, c_2=0\)), violently crashing at the boundary (\(x = -1.00\), \(J_{\mathrm{atk}} = -207.28\)). Episode 4 exhibits an identical failure mode, where an idling steering tie at step 32 (\(y=1.318\)) culminates in severe attitude rollover (\(\theta = 1.66\,\mathrm{rad}\), crashing at return \(-163.59\)). These step-level dynamics confirm that \textsc{V-BSA} operates through a fundamentally distinct pathway from update-space attacks \citep{fang2025provably}: rather than attempting to overcome robust filtering via large parameter deviations, it selectively induces behavioral paralysis at trajectory bifurcations, allowing temporal error propagation within the MDP to trigger irreversible kinematic failure.

\begin{table}[htbp]
\centering
\small
\setlength{\tabcolsep}{4.5pt}
\renewcommand{\arraystretch}{1.15}
\begin{tabular}{lccccccccc}
\toprule
& \multicolumn{3}{c}{\textbf{Dynamics} \(d_0\)} & \multicolumn{3}{c}{\textbf{Dynamics} \(d_1\)} & \multicolumn{3}{c}{\textbf{Dynamics} \(d_2\)} \\
\cmidrule(lr){2-4}\cmidrule(lr){5-7}\cmidrule(lr){8-10}
\textbf{Aggregation Rule} & \(J_{\mathrm{ctl}}\) & \(J_{\mathrm{atk}}\) & \(D\) & \(J_{\mathrm{ctl}}\) & \(J_{\mathrm{atk}}\) & \(D\) & \(J_{\mathrm{ctl}}\) & \(J_{\mathrm{atk}}\) & \(D\) \\
\midrule
FedPG-BR & 228.6 & 80.5 & 148.17 & 211.9 & 74.4 & 137.48 & 245.4 & \(-125.2\) & 370.59 \\
Trimmed Mean & 236.2 & 34.1 & 202.08 & 196.7 & \(-17.0\) & 213.62 & 209.7 & 103.5 & 106.18 \\
Median & 167.1 & \(-17.7\) & 184.82 & 233.3 & 103.6 & 129.71 & 216.5 & 3.5 & 212.95 \\
FedAvg & 243.3 & 167.3 & 76.00 & 224.7 & 112.6 & 112.10 & 194.5 & 142.8 & 51.60 \\
\bottomrule
\end{tabular}
\caption{\textbf{Discrete LunarLander-v2 evaluation breakdown under ensemble training (\(K=5, M=10\)).} Terminal evaluation returns and degradation \(D = J_{\mathrm{ctl}} - J_{\mathrm{atk}}\) across evaluation dynamics.}
\label{tab:ll_ens}
\end{table}

\begin{table}[htbp]
\centering
\small
\setlength{\tabcolsep}{4.5pt}
\renewcommand{\arraystretch}{1.15}
\begin{tabular}{lccccccccc}
\toprule
& \multicolumn{3}{c}{\textbf{Dynamics} \(d_0\)} & \multicolumn{3}{c}{\textbf{Dynamics} \(d_1\)} & \multicolumn{3}{c}{\textbf{Dynamics} \(d_2\)} \\
\cmidrule(lr){2-4}\cmidrule(lr){5-7}\cmidrule(lr){8-10}
\textbf{Aggregation Rule} & \(J_{\mathrm{ctl}}\) & \(J_{\mathrm{atk}}\) & \(D\) & \(J_{\mathrm{ctl}}\) & \(J_{\mathrm{atk}}\) & \(D\) & \(J_{\mathrm{ctl}}\) & \(J_{\mathrm{atk}}\) & \(D\) \\
\midrule
FedPG-BR & 162.7 & \(-59.2\) & 221.91 & 240.4 & \(-115.1\) & 355.48 & 225.1 & \(-61.5\) & 286.64 \\
Trimmed Mean & 258.4 & \(-12.3\) & 270.70 & 249.1 & \(-15.5\) & 264.60 & 233.9 & 189.1 & 44.80 \\
Median & 234.4 & \(-287.2\) & 521.60 & 210.3 & \(-168.1\) & 378.40 & 225.8 & \(-177.2\) & 403.00 \\
FedAvg & 216.3 & 255.0 & \(-38.70\) & 198.2 & 269.1 & \(-70.90\) & 201.5 & 263.4 & \(-61.90\) \\
\bottomrule
\end{tabular}
\caption{\textbf{Discrete LunarLander-v2 evaluation breakdown under non-ensemble training (\(N=30, M=10\)).} Terminal returns and degradation \(D\) across evaluation dynamics.}
\label{tab:ll_nbe}
\end{table}

\subsection{Analysis of Timing Ablations}
\label{app:dynamics:when}

To isolate the effect of state selection, we compare entropy gating against alternative timing schedules under identical \textsc{how} construction on LunarLander with FedPG-BR:
\begin{itemize}[leftmargin=1.5em, itemsep=1pt]
\item \textbf{Entropy Gating (\(\tau = 0.9\)):} Operates at an empirical intervention rate of \(\rho \approx 5.7\%\), producing terminal degradation of \(D = \{148.17, 137.48, 370.59\}\) across dynamics (mean \(\bar{D} = 218.75\), median \(148.17\)).
\item \textbf{Dense Poisoning (\(\rho = 100\%\)):} Intervenes at every rollout step, yielding \(D = \{-4.28, 36.65, 49.78\}\) (mean \(\bar{D} = 27.38\), median \(36.65\)).
\item \textbf{Uniform Random Selection (\(\rho = 15\%\)):} Selects states uniformly with probability \(p = 0.15\), yielding \(D = \{-15.89, 50.01, 273.79\}\) (mean \(\bar{D} = 102.64\), median \(50.01\)).
\item \textbf{Budget-Matched Bernoulli Selection (\(\rho \approx 5.7\%\)):} Matches the empirical intervention rate of entropy gating via uniform random selection on dynamics \(d_0\), yielding \(D = 91.99\) (compared to \(148.17\) for entropy gating).
\end{itemize}
These comparisons show that, with the update construction mechanism held fixed, varying the trajectory-level intervention schedule materially alters downstream policy degradation.

\subsection{Complete Convergence Trajectories}
\label{app:dynamics:curves}

Figures~\ref{fig:ens_cp}--\ref{fig:nbe_ll} provide the complete rolling evaluation trajectories over all communication rounds and evaluated task--aggregation configurations. In each plot, solid curves denote attacked training under \textsc{V-BSA}, while dashed curves denote the corresponding clean baseline. Beyond the terminal degradation reported in the main results, these trajectories expose how performance divergence develops over training, specifically distinguishing progressive, sustained performance degradation from transient endpoint fluctuations. 

Across configurations, the trajectories are consistent with the task- and aggregation-dependent boundaries reported in Section~\ref{sec:exp_results}. Settings with large terminal degradation generally exhibit an early, sustained separation between attacked and clean trajectories, whereas configurations with small or negative \(D\) display tightly tracking or non-persistent gaps. Across all panels, curves display the 10-round rolling mean of evaluation returns, with the terminal points at round \(T=261\) matching the evaluation returns used to compute the degradation statistics in Table~\ref{tab:main_results} (with complete per-dynamics breakdowns provided in Tables~\ref{tab:ll_ens} and~\ref{tab:ll_nbe}).

\vspace{-0.5em}

\begin{figure}[htbp]
\centering
\begin{subfigure}[t]{0.48\linewidth}
  \centering
  \includegraphics[width=\linewidth]{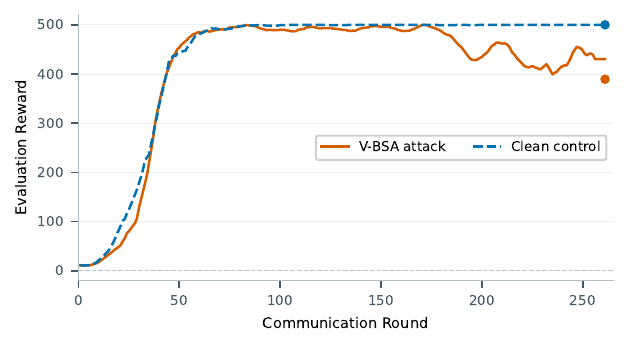}
  \caption{FedPG-BR}
  \label{fig:ens_cp_fedpgbr}
\end{subfigure}\hfill
\begin{subfigure}[t]{0.48\linewidth}
  \centering
  \includegraphics[width=\linewidth]{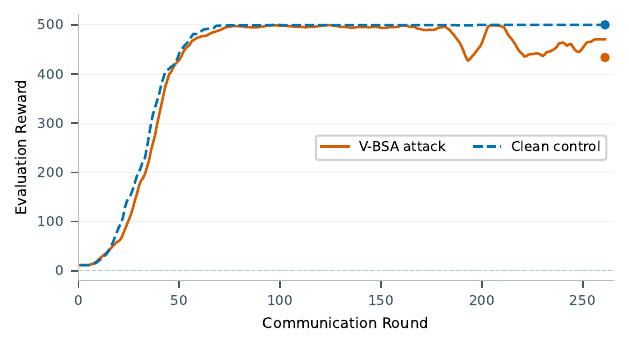}
  \caption{Trimmed Mean}
  \label{fig:ens_cp_enst}
\end{subfigure}

\vspace{0.5em}
\begin{subfigure}[t]{0.48\linewidth}
  \centering
  \includegraphics[width=\linewidth]{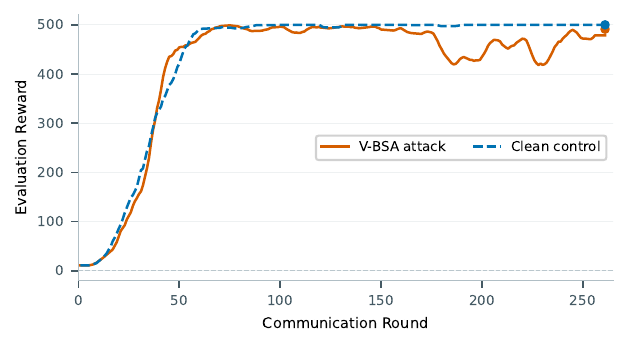}
  \caption{Median}
  \label{fig:ens_cp_ensm}
\end{subfigure}\hfill
\begin{subfigure}[t]{0.48\linewidth}
  \centering
  \includegraphics[width=\linewidth]{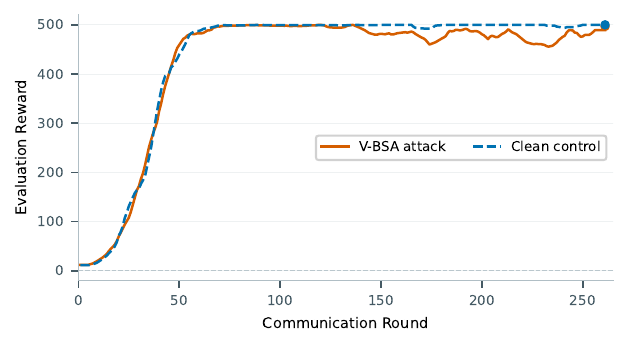}
  \caption{FedAvg}
  \label{fig:ens_cp_fedavg}
\end{subfigure}
\vspace{-0.3em}
\caption{\textbf{Ensemble paradigm on CartPole (\(K=5, M=10\)).} Evaluation trajectories across aggregation rules. Solid curves: \textsc{V-BSA} attack; dashed curves: clean control.}
\label{fig:ens_cp}
\end{figure}

\begin{figure}[htbp]
\centering
\begin{subfigure}[t]{0.48\linewidth}
  \centering
  \includegraphics[width=\linewidth]{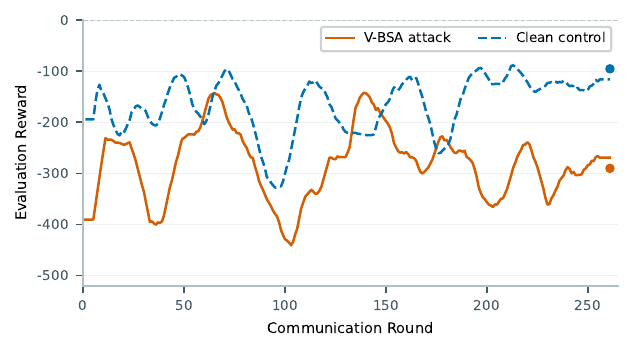}
  \caption{FedPG-BR}
  \label{fig:ens_acr_fedpgbr}
\end{subfigure}\hfill
\begin{subfigure}[t]{0.48\linewidth}
  \centering
  \includegraphics[width=\linewidth]{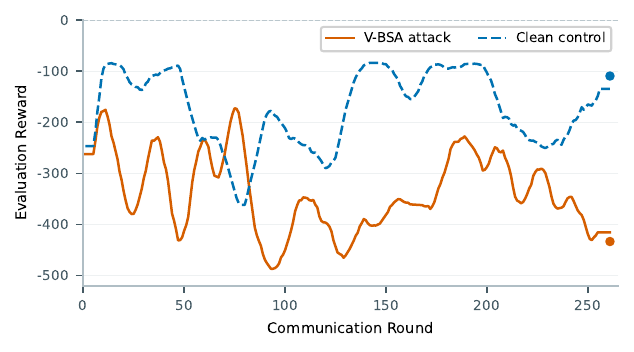}
  \caption{Trimmed Mean}
  \label{fig:ens_acr_enst}
\end{subfigure}

\vspace{0.5em}
\begin{subfigure}[t]{0.48\linewidth}
  \centering
  \includegraphics[width=\linewidth]{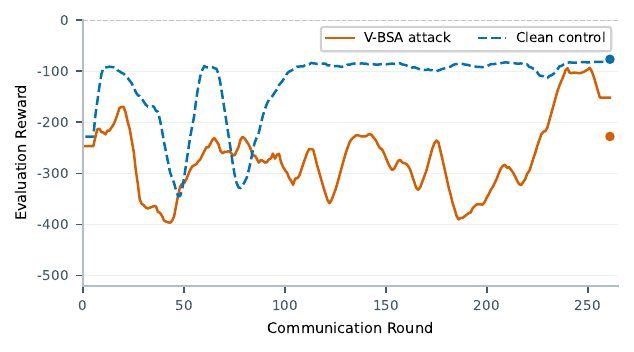}
  \caption{Median}
  \label{fig:ens_acr_ensm}
\end{subfigure}\hfill
\begin{subfigure}[t]{0.48\linewidth}
  \centering
  \includegraphics[width=\linewidth]{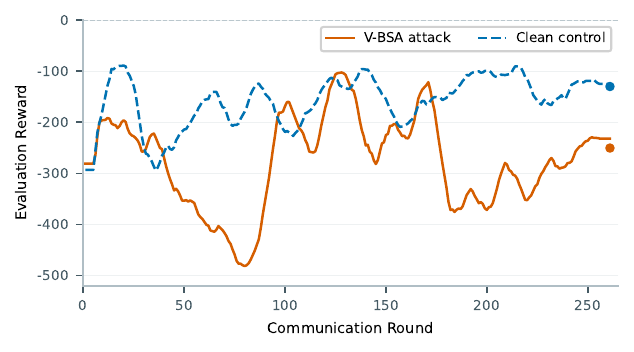}
  \caption{FedAvg}
  \label{fig:ens_acr_fedavg}
\end{subfigure}
\vspace{-0.3em}
\caption{\textbf{Ensemble paradigm on Acrobot (\(K=5, M=10\)).} Evaluation trajectories across aggregation rules. Solid curves: \textsc{V-BSA} attack; dashed curves: clean control.}
\label{fig:ens_acr}
\end{figure}

\begin{figure}[htbp]
\centering
\begin{subfigure}[t]{0.48\linewidth}
  \centering
  \includegraphics[width=\linewidth]{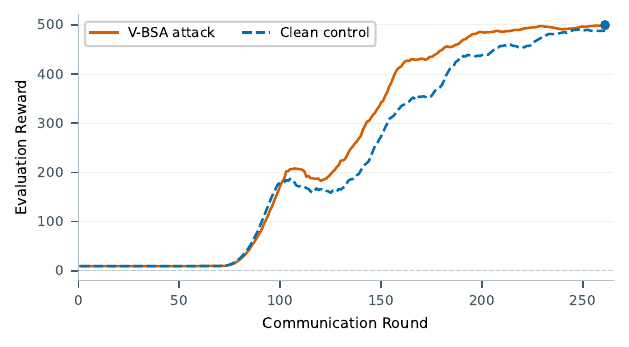}
  \caption{FedAvg}
  \label{fig:nbe_cp_favg}
\end{subfigure}\hfill
\begin{subfigure}[t]{0.48\linewidth}
  \centering
  \includegraphics[width=\linewidth]{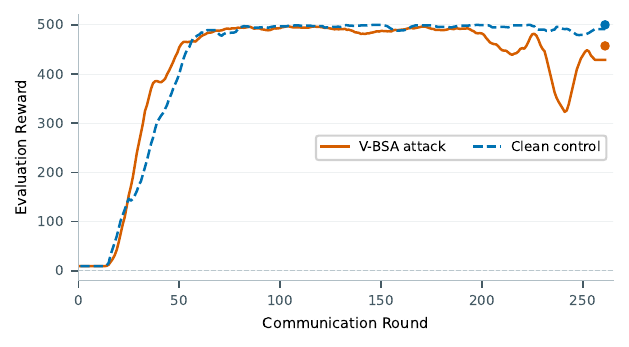}
  \caption{Median}
  \label{fig:nbe_cp_med}
\end{subfigure}

\vspace{0.5em}
\begin{subfigure}[t]{0.48\linewidth}
  \centering
  \includegraphics[width=\linewidth]{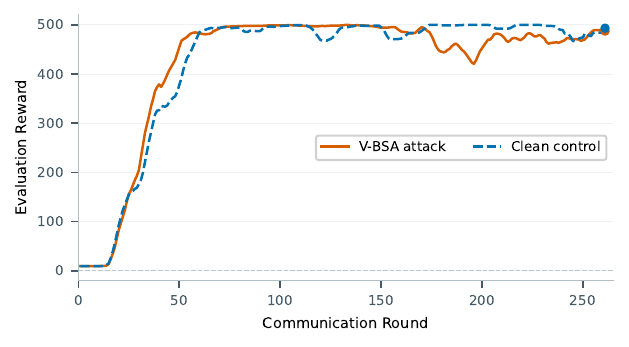}
  \caption{Trimmed Mean}
  \label{fig:nbe_cp_trm}
\end{subfigure}\hfill
\begin{subfigure}[t]{0.48\linewidth}
  \centering
  \includegraphics[width=\linewidth]{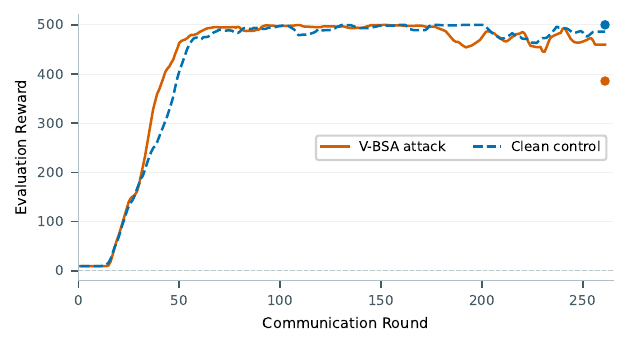}
  \caption{FedPG-BR}
  \label{fig:nbe_cp_fpbr}
\end{subfigure}
\vspace{-0.3em}
\caption{\textbf{Non-ensemble paradigm on CartPole (\(N=30, M=10\)).} Evaluation trajectories across aggregation rules. Solid curves: \textsc{V-BSA} attack; dashed curves: clean control.}
\label{fig:nbe_cp}
\end{figure}

\begin{figure}[htbp]
\centering
\begin{subfigure}[t]{0.48\linewidth}
  \centering
  \includegraphics[width=\linewidth]{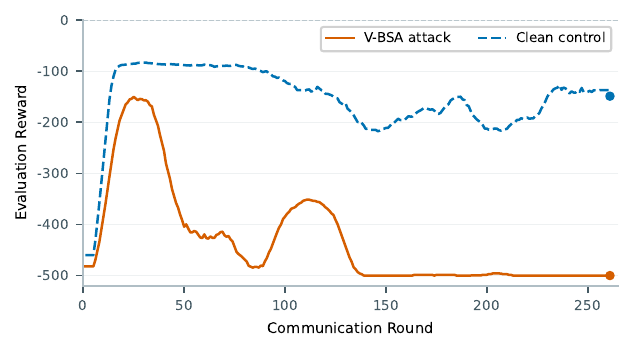}
  \caption{FedAvg}
  \label{fig:nbe_acr_favg}
\end{subfigure}\hfill
\begin{subfigure}[t]{0.48\linewidth}
  \centering
  \includegraphics[width=\linewidth]{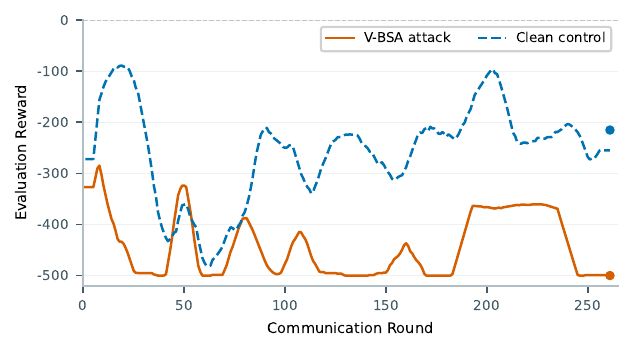}
  \caption{Median}
  \label{fig:nbe_acr_med}
\end{subfigure}

\vspace{0.5em}
\begin{subfigure}[t]{0.48\linewidth}
  \centering
  \includegraphics[width=\linewidth]{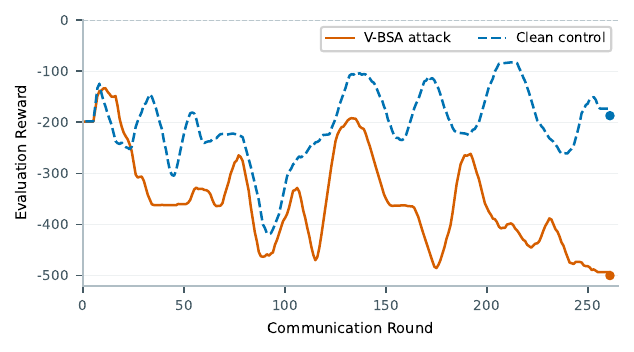}
  \caption{Trimmed Mean}
  \label{fig:nbe_acr_trm}
\end{subfigure}\hfill
\begin{subfigure}[t]{0.48\linewidth}
  \centering
  \includegraphics[width=\linewidth]{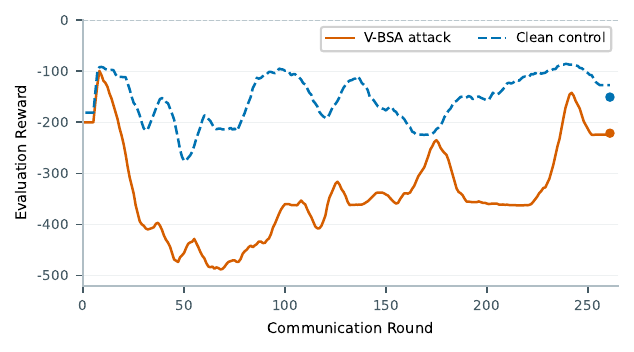}
  \caption{FedPG-BR}
  \label{fig:nbe_acr_fpbr}
\end{subfigure}
\vspace{-0.3em}
\caption{\textbf{Non-ensemble paradigm on Acrobot (\(N=30, M=10\)).} Evaluation trajectories across aggregation rules. Solid curves: \textsc{V-BSA} attack; dashed curves: clean control.}
\label{fig:nbe_acr}
\end{figure}

\begin{figure}[htbp]
\centering
\begin{subfigure}[t]{0.48\linewidth}
  \centering
  \includegraphics[width=\linewidth]{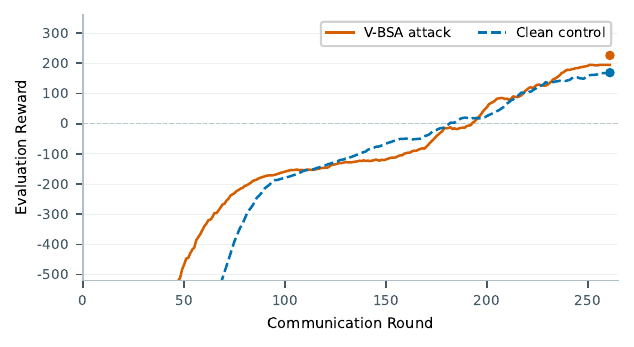}
  \caption{FedAvg}
  \label{fig:nbe_ll_favg}
\end{subfigure}\hfill
\begin{subfigure}[t]{0.48\linewidth}
  \centering
  \includegraphics[width=\linewidth]{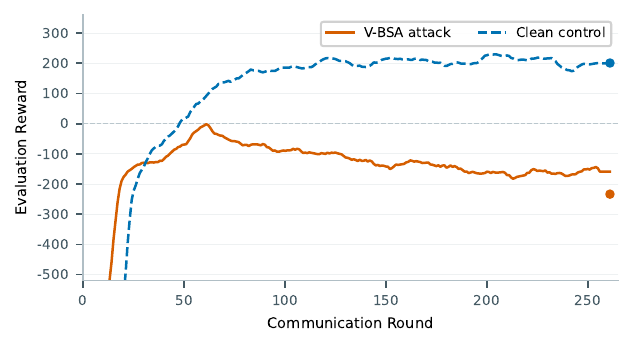}
  \caption{Median}
  \label{fig:nbe_ll_med}
\end{subfigure}

\vspace{0.5em}
\begin{subfigure}[t]{0.48\linewidth}
  \centering
  \includegraphics[width=\linewidth]{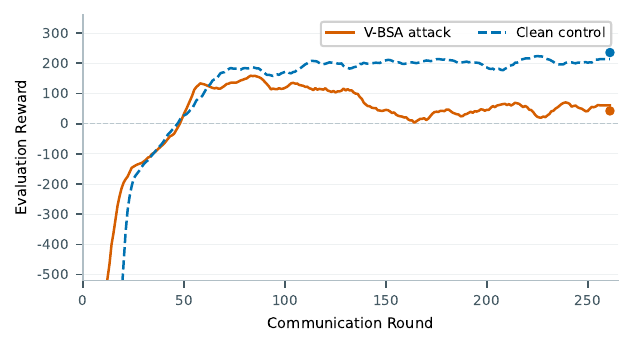}
  \caption{Trimmed Mean}
  \label{fig:nbe_ll_trm}
\end{subfigure}\hfill
\begin{subfigure}[t]{0.48\linewidth}
  \centering
  \includegraphics[width=\linewidth]{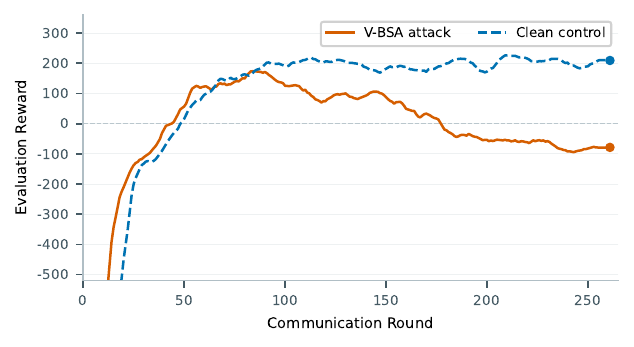}
  \caption{FedPG-BR}
  \label{fig:nbe_ll_fpbr}
\end{subfigure}
\vspace{-0.3em}
\caption{\textbf{Non-ensemble paradigm on discrete LunarLander-v2 (\(N=30, M=10\)).} Evaluation trajectories across aggregation rules. Solid curves: \textsc{V-BSA} attack; dashed curves: clean control.}
\label{fig:nbe_ll}
\end{figure}

\clearpage

\section{Exploratory Study on Continuous Action Spaces}
\label{app:continuous_study}

To examine whether the state-selective intervention framework extends to continuous control, we conduct an exploratory study on LunarLanderContinuous-v2, which shares the underlying physics and reward specification of LunarLander-v2 but operates over continuous action bounds \(\mathcal{A} = \operatorname{Box}(2)\). Policies parameterize a Gaussian distribution \(\pi_\theta(a \mid s) = \mathcal{N}(\mu_\theta(s), \Sigma)\).

\subsection{Limitations of Continuous Uncertainty Gating}
\label{app:cont:gate_saturation}

In discrete action spaces, categorical policy entropy naturally quantifies dispersion across competing discrete choices. Extending this principle to continuous Gaussian policies encounters structural limitations:
\begin{itemize}[leftmargin=1.5em, itemsep=1pt]
\item \textbf{Spatial Invariance under State-Independent Variance:} When policy variance \(\Sigma = \operatorname{Diag}(\sigma^2)\) is parameterized independently of the input state, the differential entropy \(H(\pi_\theta(\cdot \mid s)) = \frac{1}{2} \sum_{i=1}^{d_a} \log(2\pi e \sigma_i^2)\) is constant across all states, providing no discriminatory signal for trajectory-level state selection.
\item \textbf{Gate Saturation with State-Dependent Heads:} When employing a state-conditioned variance head \(\sigma_\theta(s)\), exploratory variance remains elevated throughout the state space during training. Thresholding on local variance triggers on nearly all visited trajectory steps (\(\rho \to 1.0\)), collapsing sparse state gating back into dense poisoning.
\end{itemize}

\subsection{Continuous Behavioral Steering}
\label{app:cont:steering_filtering}

In continuous action spaces, behavioral steering cannot evaluate categorical cross-entropy. Instead, we define continuous steering by perturbing predicted action means toward saturated control corners:
\begin{equation}
\mathcal{L}_{\mathrm{cont\_steer}} = \frac{1}{|\mathcal{D}_{\mathrm{when}}|} \sum_{s \in \mathcal{D}_{\mathrm{when}}} \big\| \mu_\theta(s) - a_{\mathrm{corner}}^{(i)} \big\|_2^2,
\label{eq:cont_steer}
\end{equation}
where \(a_{\mathrm{corner}}^{(i)} \in \{-0.8, +0.8\}^{d_a}\) assigns fixed control quadrants to compromised groups.

\begin{table}[h]
\centering
\small
\renewcommand{\arraystretch}{1.15}
\begin{tabular}{lcccc}
\toprule
\textbf{Perturbation Schedule} & \textbf{Dynamics} \(d_0\) & \textbf{Dynamics} \(d_1\) & \textbf{Dynamics} \(d_2\) & \textbf{Median} \(D\) \\
\midrule
Fixed Perturbation Scale (\(\lambda = 0.03\)) & \(-23.12\) & \(-169.51\) & \(+29.35\) & \(-23.12\) \\
Adaptive Scale Servo (\(\lambda \in [0.03, 0.12]\)) & \(+82.05\) & \(+82.91\) & \(-33.72\) & \(+82.05\) \\
\bottomrule
\end{tabular}
\caption{\textbf{Continuous control evaluations on LunarLanderContinuous-v2 under FedPG-BR.} Terminal degradation \(D\) across evaluation dynamics under fixed and adaptive perturbation scaling.}
\label{tab:continuous_results}
\end{table}

Table~\ref{tab:continuous_results} reports terminal degradation under fixed and adaptive scaling:
\begin{itemize}[leftmargin=1.5em, itemsep=1pt]
\item \textbf{Fixed Scaling:} A conservative perturbation scale (\(\lambda = 0.03\)) remains within local update bounds but does not produce consistent policy degradation, resulting in a negative median degradation (\(D = -23.12\)).
\item \textbf{Adaptive Scaling:} An adaptive servo adjusting perturbation scale based on local gradient acceptance yields positive degradation on dynamics \(d_0\) (\(D = 82.05\)) and \(d_1\) (\(D = 82.91\)). However, dynamics \(d_2\) exhibits negative degradation (\(D = -33.72\)), demonstrating substantial sensitivity to underlying transition dynamics.
\end{itemize}

\paragraph{Discussion.}
These exploratory results indicate that while continuous behavioral steering is technically feasible, directly applying output-space uncertainty proxies does not reliably yield sparse state selection. Developing principled susceptibility proxies for continuous control—such as measures based on representation sensitivity or local value gradients—remains an open direction.

\end{document}